\documentclass{article}
\usepackage[margin=3.2cm]{geometry}

\usepackage{booktabs} 
\usepackage{hyperref}
\usepackage{amsmath,xfrac,stackengine,amsthm}
\usepackage{graphicx}
\usepackage{amssymb}
\usepackage{dsfont} 
\usepackage{xcolor}
\usepackage[font=small,skip=2pt]{caption}
\usepackage[english]{babel}
\usepackage{enumitem}
\usepackage{subfig,float}
\usepackage{algorithm,algpseudocode} 
\usepackage{xfrac}

\usepackage[title]{appendix}

\renewcommand{\paragraph}[1]{\vspace{2mm}\noindent\textbf{#1}}
\newcommand{\diag}[1]{\text{diag}(#1)}

\newcommand{\blkdiag}[1]{\text{blkdiag}\left(#1\right)}
\newcommand{\spanning}[1]{\text{span}(#1)}
\newcommand{\dimension}[1]{\text{dim}(#1)}
\newcommand{\norm}[1]{\left\lVert#1\right\rVert}

\newcommand{\defequal}{\overset{\Delta}{=}}
\newcommand{\1}{{1}}
\newcommand{\0}{{0}}

\DeclareMathOperator*{\argmin}{arg\,min}
\newcommand{\Prob}[1]{{{P}\hspace{-1px}\left(#1\right)}}
\newcommand{\E}[1]{\mathbf{E}\hspace{-1px}\left[#1\right]}

\newcommand{\Rbb}{\mathbb{R}}

\renewcommand{\deg}{d} 
\newcommand{\davg}{\deg_{\text{avg}}}
\newcommand{\sintheta}[2]{\sin{\hspace{0.0px}\Theta\hspace{-0.0px}\big( #1, #2 \big)}}

\renewcommand{\c}[1]{{#1}_{\text{c}}} 
\newcommand{\p}[1]{\widetilde{#1}}

\newcommand{\rhomin}{\varrho_{\text{min}}} 
\newcommand{\rhomax}{\varrho_{\text{max}}} 

\newcommand{\Pmax}{P_{\text{max}}}
 
\newcommand{\alg}{\textsf{REC}} 
\newcommand{\kmeans}[3]{\mathcal{F}_{#1}\hspace{-0.5px}(#2,#3)}

\newtheoremstyle{style}
{6pt} 
{4pt} 
{\itshape} 
{} 
{\bfseries} 
{.} 
{.5em} 
{} 

\theoremstyle{style}
\newtheorem{theorem}{Theorem}[section]
\newtheorem{definition}{Definition}

\newtheorem{lemma}{Lemma}[section]

\newtheorem{remark}{Remark}
\newtheorem{property}{Property}
\newtheorem{corollary}{Corollary}[section]

\title{Spectrally approximating large graphs with smaller graphs}

\author{Andreas Loukas, Pierre Vandergheynst \\
\'{E}cole Polytechnique F\'{e}d\'{e}rale Lausanne, Switzerland}
\date{ }

\begin{document}

\maketitle


\begin{abstract}

How does coarsening affect the spectrum of a general graph?  
We provide conditions such that the principal eigenvalues and eigenspaces of a coarsened and original graph Laplacian matrices are close. The achieved approximation is shown to depend on standard graph-theoretic properties, such as the degree and eigenvalue distributions, as well as on the ratio between the coarsened and actual graph sizes. Our results carry implications for learning methods that utilize coarsening. For the particular case of spectral clustering, they imply that coarse eigenvectors can be used to derive good quality assignments even without refinement---this phenomenon was previously observed, but  lacked formal justification.
\end{abstract}


\vspace{-4mm}
\section{Introduction}

One of the most wide-spread techniques for sketching graph-structured data is coarsening. As with most sketching methods, instead of solving a large graph problem in its native domain, coarsening involves solving an akin problem of reduced size at a lower cost; the solution can then be inexpensively lifted and refined in the native domain.

The benefits of coarsening are well known both in the algorithmic and machine learning communities. 
There exists a long list of algorithms that utilize it for partitioning~\cite{hendrickson1995multi,karypis1998fast,kushnir2006fast,dhillon2007weighted,wang2014partition} and visualizing~\cite{koren2002fast,walshaw2006multilevel} large graphs in a computationally efficient manner. 
In addition, it has been frequently used to create multi-scale representations of graph-structured data, such as coarse-grained diffusion maps~\cite{lafon2006diffusion}, multi-scale wavelets~\cite{gavish2010multiscale} and pyramids~\cite{shuman2016multiscale}.  

More recently, coarsening is employed as a component of graph convolutional networks analogous to pooling~\cite{bruna2014spectral,defferrard2016convolutional,7974879}. 
Combining the values of adjacent vertices reduces the spatial size of each layer's output, prevents overfitting, and encourages a hierarchical scaling of representations.

Yet, much remains to be understood about the properties and limitations of graph coarsening. 

The majority of theoretical work has so far focused on constructing fast linear solvers using multigrid techniques. These methods are especially relevant for approximating the solution of differential equations on grids and finite-element meshes. Multigrids were also adapted to arbitrary graphs by Koutis et al.~\cite{koutis2011combinatorial} and later on by Livne and Brandt~\cite{livne2012lean}. Based on an optimized version of the Galerkin coarsening, the authors demonstrate an algebraic multi-level approximation scheme that is shown to solve symmetric diagonally dominant linear systems in almost linear time. Similar techniques have also been applied for approximating the Fiedler vector~\cite{urschel2014cascadic,gandhi2016improvement} and solving least-squares problems of the graph Laplacian~\cite{hirani2015graph,colley2017algebraic}. 

Despite this progress, with the exception of certain interlacing results~\cite{chung1997spectral,chen2004interlacing}, it is currently an open question how coarsening affects the spectrum of a general graph. As a consequence, there is no rigorous way of determining to what extend one may coarsen a graph without significantly affecting the performance of spectral methods for graph partitioning and visualization. Moreover, lacking a fundamental understanding of what and how much information is lost, 
we have little hope of equipping coarsening-based learning algorithms, such as graph neural networks, with the appropriate mechanism of constructing multi-scale representations.  

This paper sheds light into some of these questions. Specifically, we consider a one-shot coarsening operation and ask how much it affects the eigenvalues and eigenvectors of the graph Laplacian.  
Key to our argument is the introduced \emph{restricted spectral similarity} (RSS) property, asserting that the Laplacian of the coarsened and actual graphs behave similarly (up to some constants) with respect to an appropriate set of vectors. 
The RSS property is shown to hold for coarsenings constructed by contracting the edges contained in a randomized matching. Moreover, the attained constant depends on the degree distribution and can be controlled by the ratio of the coarsened and actual graph sizes, i.e., the extend of dimensionality reduction. 

We utilize the RSS property to provide spectrum approximation guarantees. It is proven that the principal eigenvalues and eigenspaces of the coarsened and actual Laplacian matrices are close when the RSS constants are not too large. Our results carry implications for non-linear methods for data clustering~\cite{von2007tutorial} and dimensionality reduction~\cite{belkin2003laplacian}. A case in point is spectral clustering: we show that lifted eigenvectors can be used to produce clustering assignments of good quality even without refinement. This phenomenon has been observed experimentally~\cite{karypis1998fast,dhillon2007weighted}, but up to now lacked formal justification. 

\vspace{0mm}
\textbf{Paper organization.} After introducing the RSS property in Section~\ref{sec:analysis}, we demonstrate in Section~\ref{sec:algorithm} how to generate coarsenings featuring small RSS constants. Sections~\ref{sec:spectrum} and~\ref{sec:clustering} then link our results to spectrum preservation and spectral clustering, respectively. The paper concludes by briefly discussing the limitations of our analysis. The proofs can be found in the appendix.

\section{Graph coarsening}
\label{sec:analysis}

Consider a weighted graph $G = (\mathcal{V}, \mathcal{E}, W)$ of $N = |\mathcal{V}|$ vertices and $M = |\mathcal{E}|$ edges, with the edge $e_{ij}$ between vertices $v_i$ and $v_j$ weighed by $w_{ij}\leq 1$. As usual, we denote by $L$ the combinatorial Laplacian of $G$ defined as
\begin{align}
L(i,j) = 
\begin{cases}
\deg_i & \mbox{if}\ i = j \\
-w_{ij} & \mbox{if}\ e_{ij} \in \mathcal{E} \\
0 & \mbox{otherwise}
\end{cases}
\end{align}
and $\deg_i$ the weighted degree of $v_i$. Moreover, let $\lambda_k$ be the $k$-th eigenvalue of $L$ and $x_k$ the associated eigenvector.

\subsection{How to coarsen a graph?}

At the heart of a coarsening lies a surjective (and therefore dimension reducing) mapping $\varphi : \mathcal{V} \rightarrow \c{\mathcal{V}}$ between the original vertex set $\mathcal{V} = \{v_1, \ldots, v_N\}$ and the smaller vertex set $\c{\mathcal{V}} = \{ u_1, \ldots, u_{n}\}$. 
In other words, the coarse graph $\c{G} = (\c{\mathcal{V}}, \c{\mathcal{E}})$ has $m = |\c{\mathcal{E}}|$ and contains every edge $(i,j) \in \mathcal{E}$ for which $\varphi(v_i) \neq \varphi(v_j)$. 
We define the \emph{coarsened Laplacian} as 
\begin{align}
	\c{L} = C L C^\top,
	\label{eq:contraction}
\end{align} 
where the fat $n \times N$ coarsening matrix $C$ describes how different $v \in \mathcal{V}$ are mapped onto the vertex set $\c{\mathcal{V}}$. Similarly, we may \emph{downsample} a vector $x\in \Rbb^N$ supported on $\mathcal{V}$ by the linear transformation
\begin{align}
	\c{x} = C x,
\end{align}
where now $\c{x}\in \Rbb^n$.  We here focus on coarsenings where each vertex $v_i$ is mapped into a single $u_j$. This is equivalent to only considering coarsening matrices with block-diagonal form 
$
C = \blkdiag{c_1^\top, \ldots, c_{n}^\top}, 
$
where each $c_i^\top = [c_i(1), \ldots, c_i(n_i)]$ is the length $n_i$ coarsening weight vector associated with the $i$-th vertex $u_i$ of $\mathcal{V}_F$. In addition, we restrict our attention to constant coarsening weight vectors of unit norm $\| c_i \|_2 = 1$ having as entries $c_i(j) = n_i^{-\sfrac{1}{2}}$.

Matrix $\c{L}$ is {not} a combinatorial Laplacian matrix (e.g., $\c{L} \1 \neq \0$ for $\1$ being the all ones vector), however $\c{L}$ can take the combinatorial Laplacian form using the simple re-normalization $Q \c{L} Q$, where $Q = \diag{ C\1}$. However, we primarily focus on $\c{L}$ and not $Q \c{L} Q$. The main reason is that we are not interested in the action of $\c{L}$ in itself, but on its effect when combined with downsampling. When acting on $\c{x}$, the coarsened Laplacian regains some of the desired properties\footnote{An equivalent construction that preserves the Laplacian form defines the coarsened Laplacian and vector as $L_{c'} = Q C L C^\top Q$  and $x_{c'}= Q^{-1} C x$, respectively. The equivalence follows since $\c{x}^\top \c{L} \c{x} = x^\top C^\top Q^{-1} Q C L C^\top Q Q^{-1} C x = x_{c'}^\top L_{c'} x_{c'}$. }: for instance, for constant $c_i$ (as we assume here) one regains the desired nullspace effect (i.e., $\c{L} C \1 = \0$).

We will also utilize the notion of a coarsening frame:
\begin{definition}[Coarsening frame]
	The coarsening frame $G_F = (\mathcal{V}_F, \mathcal{E}_F, W_F)$ is the subgraph of $G$ induced by set $\mathcal{V}_F = \left\{ v_i \ |\ \exists v_j \text{ with } \varphi(v_i) = \varphi(v_j) \right\}$. 
\end{definition}
Informally, $G_F$ is the subgraph of $G$ that is coarsened (see Figure~\ref{fig:example3}).  
We say that the coarsening corresponds to an \emph{edge contraction} if no two edges of the coarsening frame are themselves adjacent---in other words, $\mathcal{E}_F$ forms a matching on $G$.

\paragraph{Lifting.} We write $\p{x} = C^\top \c{x}$ to do an approximate inverse mapping from $\c{\mathcal{V}}$ to $\mathcal{V}$, effectively \emph{lifting} the dimension from $\Rbb^n$ back to $\Rbb^N$. 
To motivate this choice notice that, even though $\Pi = C^\top C $ is not an identity matrix, it is block diagonal $$\Pi = \blkdiag{{c_1 c_1^\top}{} , \ldots, {c_n c_n^\top}}.$$ 
Moreover, $\Pi$ is an identity mapping for all vectors in its range.

\begin{property}
$\Pi = C^\top C$ is a rank $n$ projection matrix.
\label{property:P}
\end{property} 
\begin{proof}
For each block $\Pi_i$ in the diagonal of $\Pi$, we have $\Pi_i^2 = \Pi_i \Pi_i = c_i c_i^\top c_i c_i^\top = c_i c_i^\top \norm{c_i}^2 = \Pi_i$. 
 The rank of $\Pi$ is $n$ because each diagonal block $\Pi_i$ is of rank one. 
\end{proof}
Therefore, if $x$ is a vector in $\Rbb^N$ and $\c{x} = C x$ is its coarsened counterpart, then $\p{x} = C^\top C x = \Pi x$ is a locality-preserving approximation of $x$ w.r.t. graph $G$. 

\paragraph{A toy example.} Consider the example graph shown in Figure~\ref{fig:example1} 
\begin{figure}[t]
  \centering
  \vspace{0mm}
  \subfloat[$G$.]{\includegraphics[width=0.3\columnwidth]{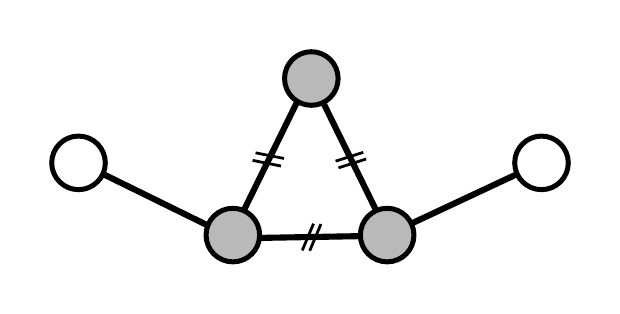}\label{fig:example1}}
  \hfill
  \subfloat[$\c{G}$.]{\includegraphics[width=0.3\columnwidth]{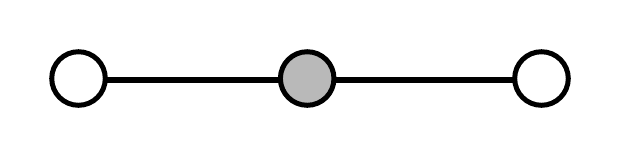}\label{fig:example2}}
  \hfill
  \subfloat[$G_F$.]{\includegraphics[width=0.3\columnwidth]{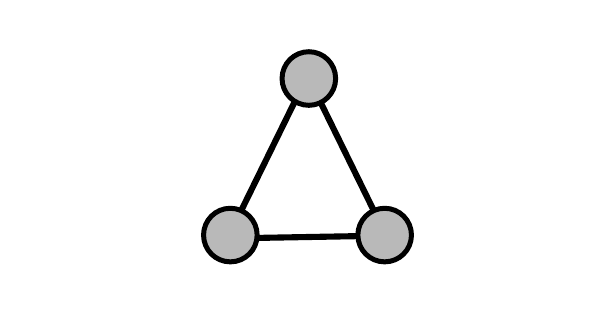}\label{fig:example3}}
 \vspace{1mm}
\caption{A toy coarsening example.\vspace{-3mm}}
\end{figure}
and suppose that we want to coarsen the $n_1 = 3$ gray vertices $\mathcal{V}_F = \{ v_1, v_2, v_3\}$ of $G$ into vertex $u_1$, as shown in Figure~\ref{fig:example2}. Matrices $C$ and $\c{L}$ take the form:
\begin{align}
 C^\top = 
\begin{bmatrix} 
1/\sqrt{3} & 1/\sqrt{3} & 1/\sqrt{3} & 0 & 0 \\ 
0 & 0 & 0  & 1 & 0 \\
0 & 0 & 0 & 0 & 1 
\end{bmatrix}
= 
\begin{bmatrix} 
c_1^\top & 0 \\ 
0 & I_{2} \notag 
\end{bmatrix}
\end{align}
\begin{align}
\c{L} = C L C^\top =
\begin{bmatrix} 
2/3 & -1/\sqrt{3} & -1/\sqrt{3} \\ 
-1/\sqrt{3}  & 1 & 0 \\
-1/\sqrt{3}  & 0 & 1 
\end{bmatrix} \notag 
\end{align}
Above, the $2 \times 2$ identity matrix $I_{2}$ preserves the neighborhood of all vertices not in $\mathcal{V}_F$. The coarsening frame is shown in Figure~\ref{fig:example3}.

\subsection{Restricted spectral similarity}

The objective of coarsening is dual. First, we aim to attain computational acceleration by reducing the dimensionality of our problem. On the other hand, we must ensure that we do not loose too much valuable information, in the sense that the structure of the reduced and original problems should be as close as possible. 

\paragraph{Spectral similarity.} One way to define how close a matrix $B$ approximates the action of matrix $A$ is to establish a spectral similarity relation of the form:
\begin{align}
(1 - \epsilon) \, x^\top A x \leq x^\top B x \leq (1 + \epsilon) \, x^\top A x,
\label{eq:spectral_similarity}
\end{align}
for all $\forall x \in \Rbb^N$ and with $\epsilon$ a positive constant. Stated in our context, \eqref{eq:spectral_similarity} can be rewritten as: 
\begin{align}
(1 - \epsilon) \, x^\top L x \leq \c{x}^\top \c{L} \c{x} \leq (1 + \epsilon) \, x^\top L x 
\label{eq:spectral_similarity2}
\end{align}
for all $ x \in \Rbb^N$ and with $\c{x} = C x$. If the equation holds, we say that matrix $\c{L}$ is an $\epsilon$-spectral approximation of $L$.
In graph theory, the objective of constructing sparse spectrally similar graphs is the main idea of spectral graph sparsifiers, a popular method for accelerating the solution of linear systems involving the Laplacian, initially proposed by Spielman and co-authors~\cite{spielman2011graph,spielman2011spectral}.

In contrast to the sparsification literature however, 
here the dimension of the space changes and one needs to take into account both the Laplacian coarsening ($L$ becomes $\c{L}$) and the vector downsampling operation ($x$ becomes $\c{x}$) in the similarity relation. Yet, from an analysis standpoint, an alternative interpretation is possible. Defining $\p{L} = \Pi  L \Pi $, we re-write 
\begin{align}
\c{x}^\top \c{L} \c{x} = {x}^\top (C^\top C) {L} (C^\top C) {x} = x^\top \Pi  L \Pi  x = x^\top \p{L} x. \notag  
\end{align}
Remembering that $C^\top$ acts as an approximate inverse of $C$, we interpret $\p{L} \in \Rbb^{N\times N}$ as an approximation of $L$ that contains the same information as $\c{L} \in \Rbb^{n\times n}$.

\paragraph{Restricted spectral similarity (RSS).} Equation~\eqref{eq:spectral_similarity2} thus states that the rank $n-1$ matrix $\p{L}$ is an $\epsilon$-spectral approximation of $L$, a matrix of rank $N-1$. Since the two matrices have different rank, the relation cannot hold for every $x \in \Rbb^N$. 
To carry out a meaningful analysis, we focus on an appropriate subset of vectors.

More specifically, we restrict our attention to the first $K$ eigenvectors of $L$ and introduce the following property:

\begin{definition}[Restricted spectral similarity]
Suppose that there exists an integer $K$ and positive constants $\epsilon_k$, such that for every $k \leq K$,
\begin{align}
	(1 - \epsilon_k) \, \lambda_k \leq x_k^\top \p{L} x_k \leq (1 + \epsilon_k) \, \lambda_k.
	\label{eq:restricted_spectral_similarity_1}
\end{align}
Then $\c{L}$ is said to satisfy the restricted spectral similarity (RSS) property with RSS constants $\{\epsilon_k\}_{k=1}^K$.
\label{def:restricted_spectral_similarity}
\end{definition}
The relation to spectral similarity is exposed by substituting $x_k^\top L x_k = \lambda_k$.

For every $k$, inequality \eqref{eq:restricted_spectral_similarity_1} should intuitively capture how close is $C x_k$ to being an eigenvector of $\c{L}$: When $\epsilon_k = 0$, vector $C x_k$ is an eigenvector of $\c{L}$ with eigenvalue $\lambda_k$. On the hand, for $\epsilon_k > 0$, $C x_k$ is not an eigenvector of $\c{L}$, but matrices $L$ and $\p{L}$ alter the length of vectors in the span of $x_k$ in a similar manner (up to $1 \pm \epsilon_k$).

This intuition turns out to be valid. In the following we will demonstrate that the RSS property is a key ingredient in characterizing the relation between the first $K$ eigenvalues and principal eigenspaces of the coarsened and actual Laplacian matrices. In particular, we will prove that the spectrum of $\c{L}$ approximates that of $L$ (up to lifting) when the constants $\epsilon_k$ are sufficiently small. This line of thinking will be developed in Section~\ref{sec:spectrum}.

\begin{remark}
	Though a uniform RSS constant $\epsilon_k \leq \epsilon$ is sufficient to guarantee spectrum preservation, we utilize the individual constants $\{\epsilon_k\}_{k=1}^K$ which lead to tighter bounds.
\end{remark}

\section{A randomized edge contraction algorithm}
\label{sec:algorithm}

Before examining the implications of the RSS property, in this section we propose an algorithm for coarsening a graph that provably produces coarsenings with bounded RSS constants $\epsilon_k$. 

The method, which we refer to as \alg, is described in Algorithm~\ref{algorithm}. \alg\, resembles the common greedy procedure of generating maximal matchings, in that it maintains a candidate set $\mathcal{C}$ of containing all edges that can be added to the matching. At each iteration, a new edge $e_{ij}$ is added and set $\mathcal{C}$ is updated by removing from it all edges in the edge neighborhood set $\mathcal{N}_{ij}$ defined as follows: $$\mathcal{N}_{ij} = \{ e_{pq} \ | \ e_{ip} \in \mathcal{E} \text{ or } e_{iq} \in \mathcal{E} \text{ or } e_{jp} \in \mathcal{E} \text{ or } e_{jq} \in \mathcal{E}\}.$$ 

Yet, \alg\, features two main differences. First, instead of selecting each new edge added to the matching uniformly at random, it utilizes a potential function $\phi$ defined on the edge set, i.e., $\phi: \mathcal{E} \rightarrow \Rbb_+$ with which it controls the probability $p_{ij}$ that every edge is contracted. The second difference is that, at each iteration, \alg\, has a chance $p_{\text{null}}$ of failing to select a valid edge (equivalently of selecting a null edge $e_{\text{null}}$). This choice is not driven by computational concerns, but facilitates the analysis: using the null edge, the probability of choosing an edge from set $\mathcal{C} \cup e_{\text{null}}$ is a valid probability measure at all times, without any need for updating the total potential $\Phi$. 

\begin{remark}
\alg\, is equivalent to the $O(M)$ complexity algorithm that samples from $\mathcal{C}$ directly in line 7 by updating $\Phi$ at every iteration such that its value is $\sum_{e_{ij} \in \mathcal{C}} \phi_{ij}$ (the condition of line 8 is thus unnecessary). Though we suggest to use this latter algorithm in practice, it is easier to express our results using the number of iterations $T$ of Algorithm~\ref{algorithm}. 
\end{remark}

\begin{algorithm}[t!]
\caption{\textsf{Randomized Edge Contraction} (\alg)}\label{algorithm}
\begin{algorithmic}[1]
\State \textbf{input}: $G = (\mathcal{V}, \mathcal{E}),\ T, \phi$
\State \textbf{output}: $\c{G}= (\c{\mathcal{V}}, \c{\mathcal{E}})$
\State $\mathcal{C} \gets \mathcal{E},\ 
\c{G} \gets G$ 
\State $\Phi \gets \sum_{e_{ij} \in \mathcal{E}} \phi_{ij}$, $t \gets 0$, $p_{\text{null}} \gets 0$.
\While{$|\mathcal{C}| > 0$ and $t < T$}
	\State $t \gets t + 1$.
	\State Select each $e_{ij}$ from $\mathcal{C}$ with prob. $p_{ij} = \phi_{ij}/ \Phi$ or fail with prob. $p_{\text{null}}$.	
	\If{$e_{ij} \in \mathcal{C}$}
		\State $\mathcal{C} \gets \mathcal{C} \setminus \mathcal{N}_{ij}$ 
		\State $p_{\text{null}} \gets p_{\text{null}} + \sum_{e_{pq} \in \mathcal{N}_{ij}} p_{pq}$
		\State $\c{G} \gets$ contract($\c{G}$, $e_{ij}$) as in~\eqref{eq:contraction}
	\EndIf
\EndWhile\label{euclidendwhile}
\end{algorithmic}
\end{algorithm}

\alg\,  returns a \emph{maximal} matching when $T$ is sufficiently large. As we will see in the following, it is sufficient to consider $T = O(N)$. The exact number of iterations will be chosen in order to balance the trade-off between the expected dimensionality reduction ratio
$$ r \defequal \E{\frac{N - n}{N}}$$
and the size of the RSS constants. %

\subsection{Analysis of \alg}
\label{subsec:analysis}

The following theorem characterizes the RSS constant of an $\c{L}$ generated by \alg\ and $L$.

\begin{theorem}
Let $\c{L}$ be the coarsened Laplacian produced by \alg\, and further suppose that
$$ \lambda_k \leq 0.5 \, \min_{ e_{ij} \in \mathcal{E}} \left\lbrace \frac{\deg_{i} + \deg_{j}}{2} + w_{ij} \right\rbrace.$$
For any $\epsilon_k \geq 0$, the relation $\lambda_k \leq x_k^\top \p{L} x_k \leq \lambda_k (1 + \epsilon_k)$ holds with probability at least
\begin{align*}
 1 - c_2 \, \frac{1 - e^{-c_1 T / N} }{4 \, \epsilon_k}  \max_{e_{ij} \in \mathcal{E}}  \chi_{ij} \left(\hspace{-0mm} \sum\limits_{e_{pq} \in \mathcal{N}_{ij}} \hspace{-1mm}\frac{w_{pq}}{w_{ij}} + 3 - \frac{4 \lambda_k}{w_{ij}} \right) 
\end{align*}
where $c_1 = N \max_{e_{ij} \in \mathcal{E}} \sum\limits_{e_{pq} \in \mathcal{N}_{ij}} p_{pq}$, $$c_2 = \frac{c_1/N}{1 - e^{-c_1/N}} \quad \text{and} \quad \chi_{ij} = \frac{\phi_{ij}}{ \sum\limits_{e_{pq} \in \mathcal{N}_{ij}} \phi_{pq}}.$$ 
\label{theorem:contraction_similarity_general}
\end{theorem}
\vspace{-4mm}
The theorem reveals that the dependency of $\epsilon_k$ to some extremal properties implied by the potential function $\phi$ and the graph structure. It is noteworthy that 
\begin{align}
	c_1 = O(1) \quad \text{implies} \quad \lim_{N \rightarrow \infty} c_2 = 1.
\end{align}
These asymptotics can be taken at face value even for finite size problems: coarsening typically becomes computationally relevant for large $N$ (typically $N > 10^3$), for which $c_2$ has effectively converged to its limit.

The assumption that $c_1$ is independent of $N$ can be satisfied either by assuming that $G$ is a \emph{bounded degree graph}, such that $|\mathcal{N}_{ij}| \ll N$ for every $e_{ij} \in \mathcal{E}$, or by choosing potential functions $\phi_{ij}$ that are inversely proportional to $|\mathcal{N}_{ij}|$.

We can also incorporate the expected reduction ratio $r$ in the bound, by noting that  
\begin{align} 
r N = \sum_{e_{ij} \in \mathcal{E}} \Prob{e_{ij} \in \mathcal{E}_F} 
&\geq \sum_{e_{ij} \in \mathcal{E}} p_{ij}  \frac{1- e^{-T P_{ij}} }{P_{ij}} \notag \\
&\hspace{-1cm}\geq \frac{1- e^{-T \Pmax} }{\Pmax } = \frac{1- e^{-c_1 T/N} }{c_1 / N },
\end{align}
(see proof of Theorem~\ref{theorem:contraction_similarity_general} for definitions of $P_{ij}$ and $\Pmax$) implying  
\begin{align}
	1- e^{-c_1 T/N} \leq r c_1, 
	\label{eq:ratio_inequality}
\end{align}
as well as that $T = \frac{N}{c_1} \log\left( \frac{1}{1 - r c_1 } \right) $ iterations suffice to achieve any $r < 1 /c_1$. Nevertheless, this latter estimate is more pessimistic than the one presented in Theorem~\ref{theorem:contraction_similarity_general}.

\paragraph{The norm $\|\Pi^\bot x_k\|_2^2$.} For all $k$, one has
\begin{align*}
	\Prob{ \|\Pi^\bot x_k\|_2^2 \geq \epsilon \, \lambda_k } \leq c_2 \frac{1 - e^{-c_1 T / N}}{2 \, \epsilon} \, \max_{e_{ij} \in \mathcal{E}} \frac{\chi_{ij}}{w_{ij}},
\end{align*}
with constants defined as before (the derivation is not included as it resembles the one employed in the proof of Theorem~\ref{theorem:contraction_similarity_general}). Thus, $\|\Pi^\bot x_k\|_2^2$ depends on the ratio $r$ (through~\eqref{eq:ratio_inequality}) and is smaller for small $k$ (due to $\lambda_k$). This is reasonable: by definition, eigenvectors corresponding to small eigenvalues are smooth functions on $G$; averaging some of their entries locally on the graph is unlikely to alter their values significantly.

\begin{figure*}[t]
  \centering
  \vspace{-1mm}
  \subfloat[Bunny (point cloud)]{\includegraphics[width=0.250\textwidth,trim=3.2cm 3.0cm 3cm 3.5cm, clip]{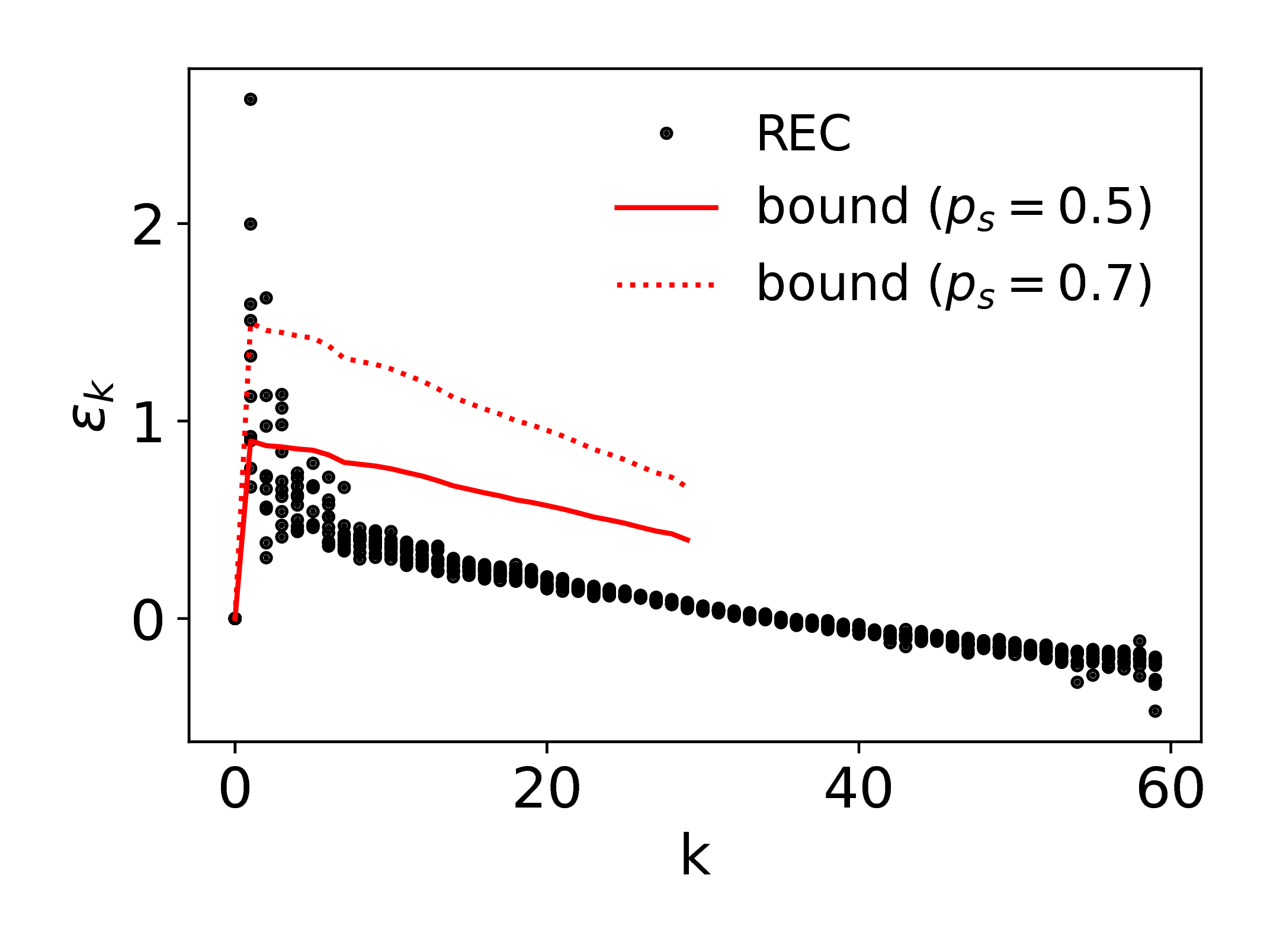}\label{fig:bunny}}
  \hfill
  \subfloat[Swiss roll (manifold)]{\includegraphics[width=0.250\textwidth,trim=3.2cm 3.0cm 3cm 3.5cm, clip]{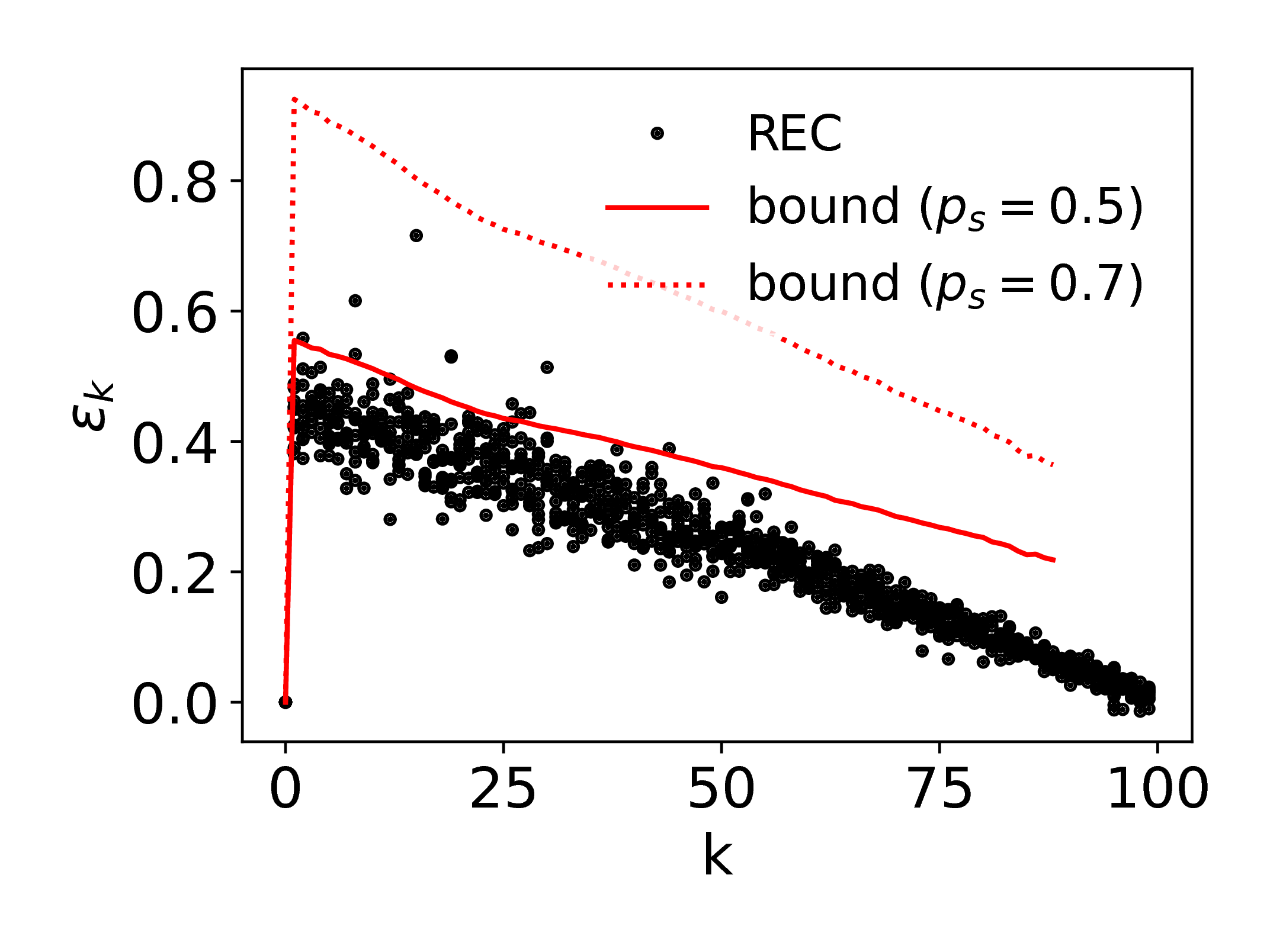}\label{fig:swissroll}}
  \hfill
  \subfloat[Yeast (protein network)]{\includegraphics[width=0.250\textwidth,trim=3.2cm 3.0cm 3cm 3.5cm, clip]{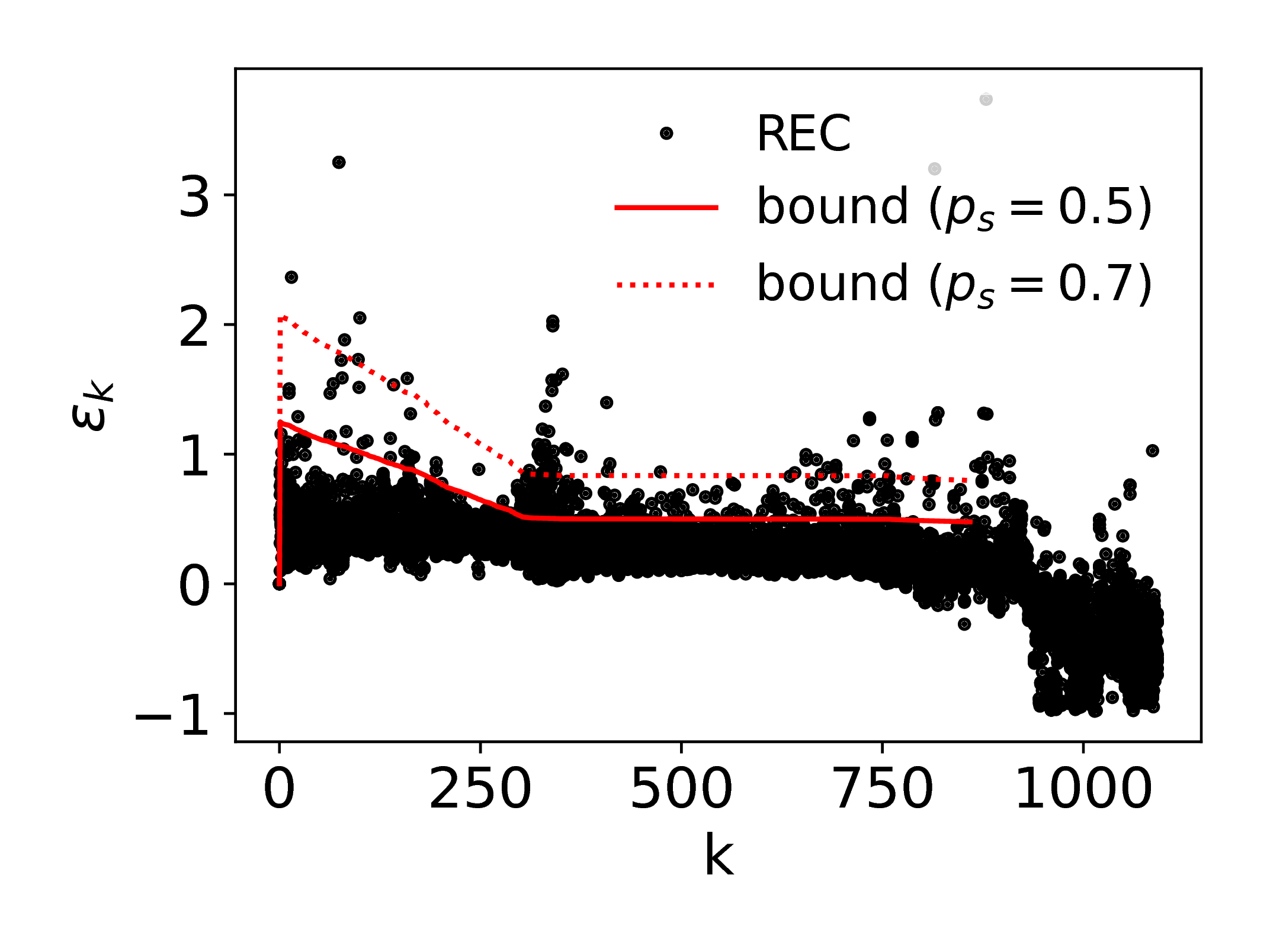}\label{fig:yeast}}
  \hfill
  \subfloat[Regular graph]{\includegraphics[width=0.250\textwidth,trim=3.2cm 3.0cm 3cm 3.5cm, clip]{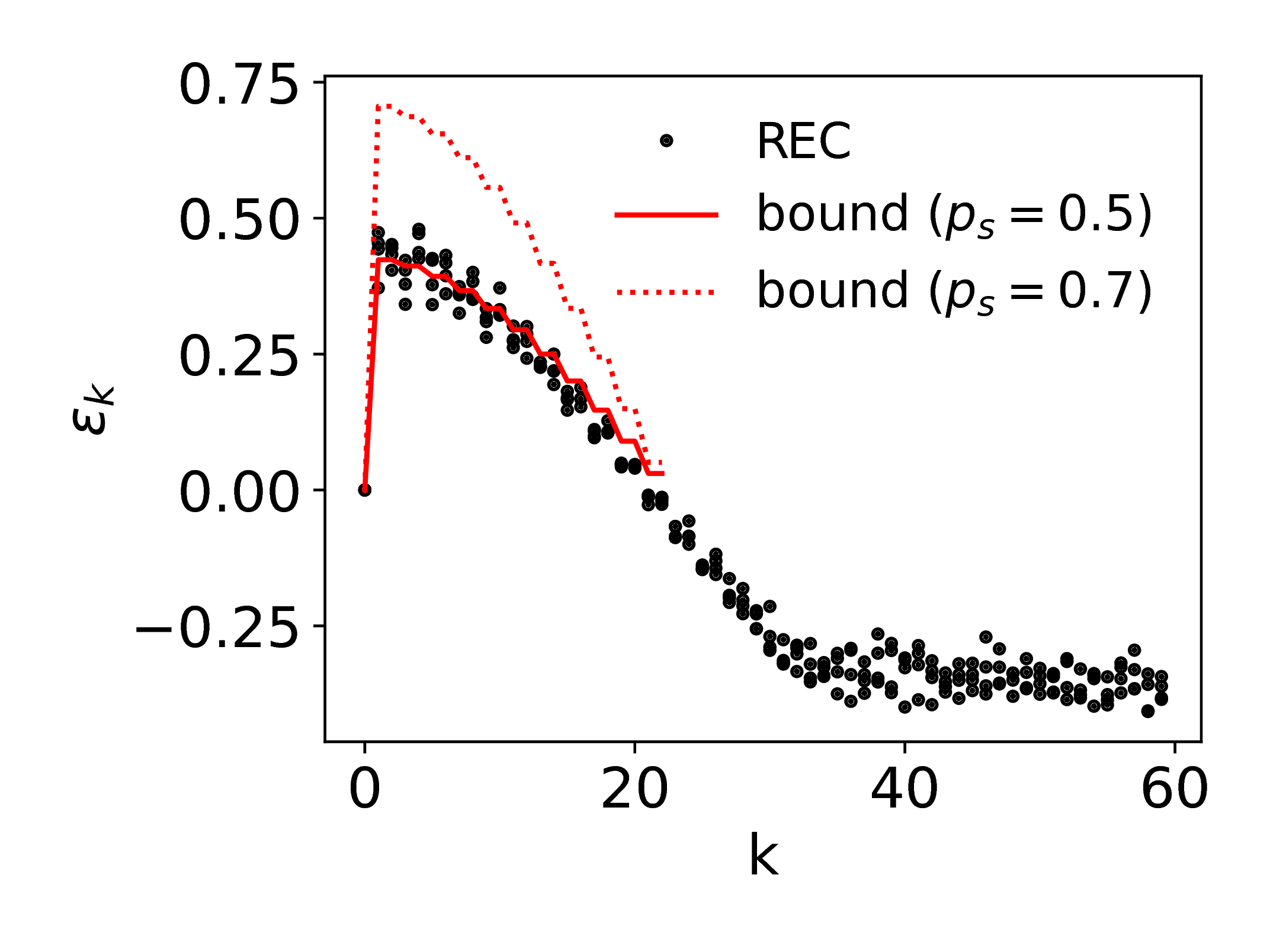}\label{fig:yeast}}  
 \vspace{2mm}
\caption{The proposed bounds follow the behavior of the RSS constants $\epsilon_k$, especially for regular graphs or graphs with small degree variance. The two red lines plot the bounds of Theorem~\ref{theorem:contraction_similarity_general} for a success probability of $p_s = 0.5$ and $p_s = 0.7$. \label{fig:RSS}\vspace{-4mm}}
\end{figure*}

\subsection{The heavy-edge potential function} 

Let us examine how the achieved results behave for a specific potential function. Setting $\phi_{ij} = w_{ij}$ is a simple way to give preference to heavy frames---indeed, heavy-edge matchings have been utilized as a heuristic for coarsening (e.g., in combination with graph partitioning~\cite{karypis1998multilevelk}). It is perhaps interesting to note this particular potential function can be derived naturally from Theorem~\ref{theorem:contraction_similarity_general} if we require that  
\begin{align}
	\chi_{ij} \, \frac{ \sum_{e_{pq} \in \mathcal{N}_{ij}} w_{pq} }{w_{ij}}  = 1 \quad \text{for all} \ e_{ij}.
\end{align}
It will be useful to denote respectively by $\rhomin$ and $\rhomax$ the minimum and maximum of expression $({d_i + d_j - w_{ij}})/{2 \davg}$ over all $e_{ij}$, with $\davg$ being the average degree. It is straightforward to calculate that in this case
$$ 
\quad c_1 = 2 \, \left( \frac{d_i + d_j - w_{ij}}{\davg}\right) = 4 \rhomax. $$ 
Therefore, $c_1 = O(1)$ for all graphs in which $\Omega(1) = \rhomin \leq \rhomax = O(1)$, and given sufficiently large $N$ and some manipulation the probability estimate of Theorem~\ref{theorem:contraction_similarity_general}  reduces to    
\begin{align}
1 - \frac{1 - e^{-4\rhomax T/N}}{ 4 \, \epsilon_k } \, \left( 1 + \frac{1.5 - 2 \lambda_k}{ \davg \, \rhomin }\right).   
\label{col:contraction_similarity_heavy}
\end{align}
In addition, 
$
	\Prob{ \|\Pi^\bot x_k\|_2^2 >  \epsilon \, \lambda_k } \leq \frac{1 - e^{-4\rhomax T/N}}{2 \, \epsilon \, \rhomin \, \davg}.
$

The heavy-edge potential function is therefore more efficient for graphs with small degree variations. Such graphs are especially common in machine learning, where often the connecticity of each vertex is explicitly constructed such that all degrees are close to some target value (e.g., using a $k$-nearest neighbor graph construction~\cite{muja2014scalable}).

As a proof of concept, Figure~\ref{fig:RSS} compares the actual constants $\epsilon_k$ with the bound of Theorem~\ref{theorem:contraction_similarity_general} when utilizing \alg\ with a heavy-edge potential to coarsen the following benchmark graphs: (\emph{i}) a point cloud representing a bunny obtained by re-sampling the Stanford bunny 3D-mesh~\cite{turk1994zippered} and applying a k-nn construction ($N = 1000, r = 0.4, k=30$), (\emph{ii}) a k-nn similarity graph capturing the geometry of a 2D manifold usually referred to as Swiss roll ($N = 1000, r=0.4, k = 10$), (\emph{iii}) A network describing the interaction of yeast proteins~\cite{nr-aaai15} ($N = 1458, r = 0.25, \davg = 2, d_\text{max} = 56$), and (\emph{iv}) a $d$-regular graph ($N = 400, r = 0.4, d = 20$). As predicted by our bounds, $\epsilon_k$ decrease with $k$ (the decrease is close to linear in $\lambda_k$) and with the variance of the degree distribution. The heavy-tailed yeast network and the regular graph constitute two extreme examples, with the latter featuring much smaller constants.

\subsection{Regular graphs} 

For regular graphs, \eqref{eq:ratio_inequality} becomes asymptotically tight leading to the following Corollary:
\begin{corollary}
If $G$ is a regular graph with combinatorial degree $d$ and equal edge weights $w_{ij} = w$, then for any $k$ such that $ \lambda_k \leq (d + 1)/2$ and for sufficiently large $N$, the relation $\lambda_k \leq x_k^\top \p{L} x_k \leq \lambda_k (1 + \epsilon_k)$ holds with probability at least
\begin{align}
	\geq 1 - r \, \frac{1 - (2d)^{-1}}{ \epsilon_k} \left(1 + \frac{1.5 - \lambda_k}{ d-0.5}\right) 
	\overset{d \gg 1}{\approx} 1 - \frac{r}{ \epsilon_k}.
\end{align}
Furthermore, inequality $\norm{\Pi^\top x_k}_2^2 \geq \epsilon r \lambda_2$ holds for all $k$ with probability at most $2 / (d \epsilon)$ and $T = \frac{N}{2(2 - 1/d)} \log\left(\frac{1}{1 - 2(2 - 1/d) r} \right)$ iterations of \alg\, suffice in expectation to achieve reduction $r$.
\label{cor:regular}
\end{corollary}

An other way to read Corollary~\ref{cor:regular} is that, for a sufficiently dense regular graph, there exists\footnote{The existence is implied by the probabilistic method.} an edge contraction for which $\c{L}$ satisfies the RSS property with constants bounded by $r$.

\section{The spectrum of the coarsened Laplacian}
\label{sec:spectrum}

This section links the RSS property with spectrum preservation. Our results demonstrate that the distance between the spectrum of a coarsened Laplacian and of the combinatorial Laplacian it approximates is directly a function the RSS constant between the two matrices. This relation also extends to eigenspaces.

\subsection{Basic facts about the spectrum}
\label{subsec:basic}

Before delving into our main results, let us first consider the spectrum of a coarsened Laplacian which does not (necessarily) meet the RSS property.
 
W.l.o.g., let $G$ be connected and sort its eigenvalues as $0 = \lambda_1 < \lambda_2 \leq \ldots \leq \lambda_N.$ Similarly, let $\p{\lambda}_k$ be the $k$-th largest eigenvalue of the coarsened Laplacian $\c{L}$ and name $\p{x}_k$ the associated eigenvector. 
As the following theorem shows, there is a direct relation between the eigenvalues $\p{\lambda}$ and $\lambda$.
\begin{theorem}
Inequality $ \lambda_k \leq \p{\lambda}_k $ holds for all $k \leq n$.
\label{theorem:interlacing}
\end{theorem}

We remark the similarity of the above to a known result in spectral graph theory~\cite{chung1997spectral} (Lemma 1.15) assering that, if $\nu_k$ is the $k$-th eigenvalue of the normalized Laplacian of $G$ and $\p{\nu}_k$ is the $k$-th eigenvalue of the normalized Laplacian of a graph $\c{G}$ obtained by edge contraction, then $\nu_k \leq \p{\nu}_k$  for all $k = 1, 2, \ldots, n$.
Despite this similarity however, Theorem~\ref{theorem:interlacing} deals with the eigenvalues of the combinatorial Laplacian matrix and its coarsened counterpart $\c{L} = C L C^\top$.

We also notice that, when $c$ is chosen to be constant over each connected component of $G_F$ (as we assume in this work) the nullspace of $\c{L}$ spans the downsampled constant vector implying that 
\begin{align}
C^\top \p{x}_1 = x_1 \quad \text{and} \quad  0=\p{\lambda}_1  < \p{\lambda}_2.
\label{eq:first_eigenvalue}
\end{align}
The above relations constitute the main reason why we utilize constant coarsening weights in our construction.

\subsection{From the RSS property to spectrum preservation}

For eigenvalues, the RSS property implies an upper bound:
\begin{theorem}
	If $\c{L}$ satisfies the RSS property, then  
	\begin{align}
		\lambda_k \leq \p{\lambda}_k \leq \max\left\lbrace \p{\lambda}_{k-1},\ \frac{(1 + \epsilon_k)}{ \sum_{i \geq k} (\p{x}_i^\top C x_k)^2} \, \lambda_k \right\rbrace
	\end{align}
	\label{theorem:eigenvalue}
for all $k \leq K$, where $\epsilon_k$ is the $k$-th RSS constant.
\end{theorem}

The term $\sum_{i \geq k} (\p{x}_i^\top C x_k)^2$ depends on the orientation of the eigenvectors of $\c{L}$ with respect to those of $L$. We expect:
\begin{align*}
	\lambda_k 
	\leq \p{\lambda}_k \leq \frac{(1 + \epsilon_k)}{ \sum_{i \geq k} (\p{x}_i^\top C x_k)^2} \, \lambda_k 
	\approx \frac{(1 + \epsilon_k)}{ \|\Pi x_k \|_2^2} \, \lambda_k. 
\end{align*}
Indeed, for $\lambda_2$ the above becomes an equality as $$\sum_{i \geq 2} (\p{x}_i^\top C x_2)^2 = \| \Pi x_2 \| -  (\p{x}_1^\top C x_2)^2 = \| \Pi x_2 \|,$$ 
where the last equality follows from~\eqref{eq:first_eigenvalue}. In this case, the above results combined with the analysis presented in Section~\ref{subsec:analysis} imply the following corollary: 
\begin{corollary}
Consider a bounded degree graph with $\lambda_2 \leq 0.5 \min_{e_{ij} \in \mathcal{E}} \left\lbrace \frac{d_i + d_j}{2} + w_{ij}\right\rbrace $  and suppose that it is coarsened by \alg\, using a heavy-edge potential. For any feasible expected dimensionality reduction ratio $r$, sufficiently large $N$ and any $\epsilon > 0$
$$ \p{\lambda}_2 \leq \frac{1 + r \epsilon}{1 - \lambda_2 \, r \epsilon} \, \lambda_2,$$
with probability at least  $1 - \frac{c_3}{4 \, \epsilon} \left( 1 + \frac{1.5 w_{\text{max}} + 2 (1-\lambda_2)}{\davg \, \rhomin} \right)$ where $c_3 = r \, (1 - e^{-4 \rhomax T/N})$. For a $d$-regular graph this probability is at least $1 - \frac{1}{\epsilon}(1 + \frac{3 - \lambda_2}{d}).$
\end{corollary}
The statement can be proved by taking a union bound with respect to the events $\{ \|\Pi^\bot x_2\|_2^2 > \lambda_2 \, r \epsilon \}$ and $\{  x_2^\top \p{L} x_2 > (1 + r \epsilon) \lambda_2\} $, whose probabilities can be easily obtained from the results of Section~\ref{sec:algorithm}. 

\textbf{Eigenspaces}. We also analyze the angle between principal eigenspaces of $L$ and $\c{L}$.  We follow Li~\cite{li1994relative} and split the (lifted) eigendecompositions of $L$ and $\c{L}$ as
\begin{align*}
L &= X \Lambda X^\top = (X_k, X_{k^\bot})
	\begin{pmatrix}
		\Lambda_k &  \\ 
		 & \Lambda_{k^\bot} 
	\end{pmatrix}
	\begin{pmatrix}
		X_k^\top  \\ 
		X_{k^\bot}^\top 
	\end{pmatrix} \\
	 C^\top \c{L} C &= (C^\top \p{X}) \p{\Lambda} (\p{X}^\top C) \\ &= (C^\top \p{X}_k, C^\top \p{X}_{k^\bot})
	 \begin{pmatrix}
	 \p{\Lambda}_k &  \\ 
	 & \p{\Lambda}_{k^\bot} 
	 \end{pmatrix}
	 \begin{pmatrix}
	 \p{X}_k^\top C  \\ 
	 \p{X}_{k^\bot}^\top C
	 \end{pmatrix},
\end{align*}
where $\Lambda_k = \diag{\lambda_1, \ldots, \lambda_k}$ and  $ X_1 = (x_1, \ldots, x_k) $ (analogously for $\p{\Lambda}_k$ and $\p{X}_k$).
The \emph{canonical angles}~\cite{davis1970rotation,stewart1990matrix} between the eigenspaces spanned by $X_k$ and $C^\top \p{X}_k$ are the singlular values of the matrix
\begin{align}
	\Theta( X_k, C^\top\p{X}_k) \defequal 
	\arccos(X_k^\top C^\top\p{X}_k \p{X}_k^\top C X_k)^{-\sfrac{1}{2}}
\end{align}
and moreover, the smaller the sinus of the canonical angles are, the closer the two subspaces lie.

The following theorem characterizes $\vartheta_k = \| \sintheta{X_k}{C^\top \p{X}_k} \|_F^2$, a measure of the miss-alignment of the eigenspaces spanned by $X_k$ and $C^\top \p{X}_k$.
\begin{theorem}
	If $\c{L}$ satisfies the RSS property, then 
	\begin{align*}
		\vartheta_k &\leq \min \left\{ \sum\limits_{2\leq i \leq k} \frac{ \epsilon_i\lambda_i + \lambda_k \| \Pi^\bot x_i \|_2^2 }{\p{\lambda}_{k+1} - \lambda_{k}}, \ \sum\limits_{2\leq i \leq k} \frac{(1 + \epsilon_i)\lambda_i - \lambda_2 \| \Pi x_i \|_2^2 }{\p{\lambda}_{k+1} - \lambda_2} \right\},
	\end{align*}
	for every $k\leq K$. 
	\label{theorem:sintheta}
\end{theorem}
Both bounds have something to offer: The first is applicable to situations where there is a significant eigenvalue separation between the subspace of interest and neighboring spaces (this condition also appears in classic perturbation analysis~\cite{davis1970rotation}) and has the benefit of vanishing when $n = N$. The second bound on the other hand does not depend on the minimum eigengap between $\lambda_{k}$ and $\lambda_{k+1}$, but on the gap between every eigenvalue $\lambda_i$ in the subspace of interest and $\lambda_{k+1}$, which can be significantly smaller.

We obtain an end-to-end analysis of coarsening by combining Theorem~\ref{theorem:sintheta} with Theorem~\ref{theorem:contraction_similarity_general} and taking a union bound over all $k\leq K$. However, the reader is urged to consider the proof of Corollary~\ref{corollary:spectral_clustering} for a more careful analysis with significantly improved probability estimates. 

\begin{figure}[t]
  \centering
  \vspace{-1mm}
  \includegraphics[width=0.63\columnwidth]{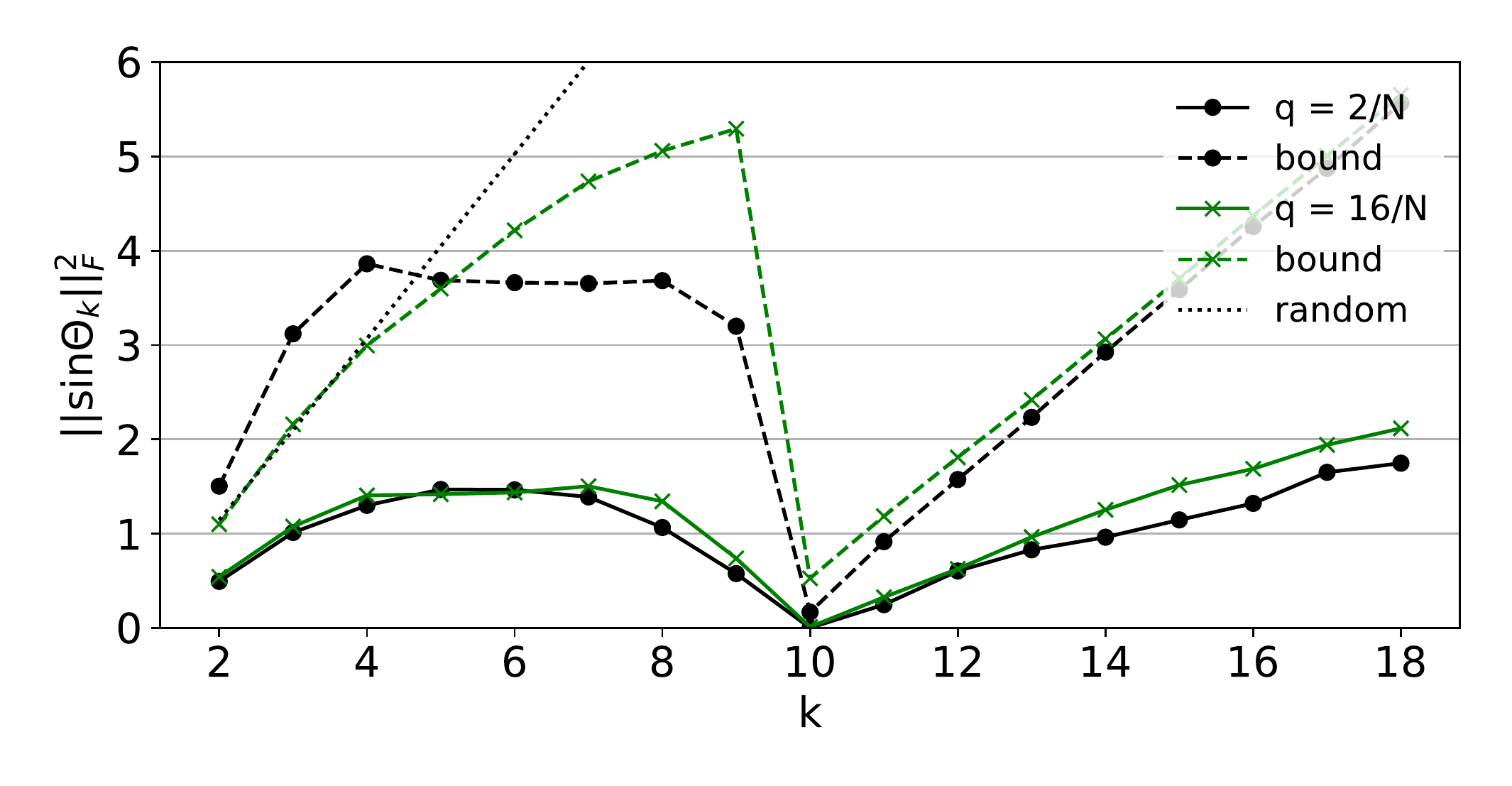}
\caption{The extend to how coarsening preserves eigenspace alignment is a function of eigenvalue distribution.\label{fig:sintheta} \vspace{-4mm}}
\end{figure}

The importance of the eigenvalue distribution can be seen in Figure~\ref{fig:sintheta}, where we examine the alignment of $X_k$ and $C^\top \p{X}_k$ for different $k$ when $r = 0.4$. The figure summarizes the results for 10 stochastic block model graphs, each consisting of $N = 1000$ vertices. These graphs were built by uniformly assigning vertices into $K=10$ communities and connecting any two vertices with probability $p$ or $q$ depending on whether they belong in the same or different communities, respectively. Such constructions are well known to produce eigenvalue distributions that feature a large gap between the $K$ and $K+1$ eigenvalues and small gaps everywhere else. 

Below $K$ the eigenspaces are poorly aligned and not much better than chance (dotted line). As soon as the size of the subspace becomes equal to $K$ however we observe a significant drop, signifying good alignment. This matches the prediction offered by our bounds (dashed line). The phenomenon is replicated for two parametrizations of the stochastic block model, one featuring a low $q$ (and thus a large gap) and one with larger $q$. Due to the smaller gap, in the latter case the trend is slightly less exaggerated.

\section{Implications for spectral clustering}
\label{sec:clustering}

Spectral clustering is a non-linear method for partitioning $N$ points $z_1, z_2, \ldots, z_N \in \Rbb^D$ into $K$ sets $S = \{S_1, S_2, \ldots, S_K\}$. There exist many versions of the algorithm. We consider the ``unnormalized spectral clustering''~\cite{von2007tutorial}: 

\begin{enumerate}[noitemsep,topsep=0pt,parsep=0pt,partopsep=0pt,leftmargin=4.5mm]
	\item Construct a similarity graph with $w_{ij} = e^{- \| z_i - z_j\|_2^2/\sigma^2}$ between vertices $v_i$ and $v_j$. Let $L$ be the combinatorial Laplacian of the graph and write $\Psi = X_K \in \Rbb^{N\times K}$ to denote the matrix of its first $K$ eigenvectors.
	\item Among all cluster assignments $\mathcal{S}$, search for the assignment $S^{*}$ that minimizes the k-means cost: 
	$$ \kmeans{K}{\Psi}{S} = \sum_{k = 1}^K \sum_{v_i, v_j \in S_k } \frac{\| \Psi(i,:) - \Psi(j,:)\|_2^2}{2\, |S_k|} $$
\end{enumerate} 

Though a naive implementation of the above algorithm scales with $O(N^3)$, the acceleration of spectral clustering has been an active topic of research. A wide-range of sketching techniques have been proposed\cite{boutsidis2015spectral,tremblay2016compressive}, arguably one of the fastest known algorithms utilizes coarsening. Roughly, the algorithm involves: (\emph{i}) hierarchically coarsening the input graph (using edge contractions) until the latter reaches a target size; (\emph{ii}) solving the clustering problem is the small dimension; (\emph{iii}) lifting the solution back to the is lifted to the original domain; and (\emph{iv}) performing some fast refinement to attain the final clustering.

In the following, we provide theoretical guarantees on the solution quality of the aforementioned scheme for a single coarsening level. To the extend of our knowledge, this is the first time that such an analysis has been carried out. 

To perform the analysis, we suppose that 
$$ S^{*} = \argmin_{S \in \mathcal{S}} \kmeans{K}{\Psi}{S} \ \  \text{and} \ \ \p{S}^{*} = \argmin_{S \in \mathcal{S}} \kmeans{K}{\p{\Psi} }{S}$$
are the (optimal) clustering assignments obtained by solving the $k$-means using as input the original eigenvectors $X_K$ and the lifted eigenvectors $\p{\Psi} = C^\top \p{X}_K$ of $\c{L}$, respectively. We then measure the quality of $\p{S}^{*}$ by examining how far the correct minimizer $\kmeans{K}{\Psi}{{S}^*}$ is to $\kmeans{K}{\Psi}{\p{S}^{*}}$. Note that the latter quantity utilizes the correct eigenvectors as points and necessarily $\kmeans{K}{\Psi}{{S}^*} \leq \kmeans{K}{\Psi}{\p{S}^{*}}$. Boutsidis et al.~\cite{boutsidis2015spectral} noted that, if the two quantities are close then, despite the assignments themselves possibly being different, they both feature the same quality with respect to the k-means objective. 

We prove the following approximation result: 
\begin{corollary}
Consider a bounded degree graph with $\lambda_K \leq 0.5 \min_{e_{ij} \in \mathcal{E}} \left\lbrace \frac{d_i + d_j}{2} + w_{ij}\right\rbrace $ and suppose that it is coarsened by \alg\, using a heavy-edge potential. For sufficiently large $N$, any feasible ratio $r$, and $\epsilon > 0$, 
$$ \left[ \kmeans{K}{\Psi}{\p{S}^*}^{\sfrac{1}{2}} - \kmeans{K}{\Psi}{S^*}^{\sfrac{1}{2}} \right]^{2} \leq \sum_{k=2}^K \frac{8 \epsilon r \lambda_k }{\delta_K} $$
with probability at least $ 1- \frac{\rhomax}{\epsilon} \left( 1 + \frac{6 + 4 \lambda_K - 8 \, c_3}{\davg \rhomin} \right)$, where $\delta_K = \lambda_{K+1} - \lambda_{K}$ and $c_3 = {\sum_{k=2}^K \lambda_k^2}/{\sum_{k=2}^K \lambda_k}$.
\label{corollary:spectral_clustering}
\end{corollary}

The theorem therefore provides conditions such that the clustering assignment produced with the aid of coarsening has quality that is close to that of the original in terms of absolute error, even {without} refinement. Practically, our result states that $\p{S}^{*}$ is a good candidate for the final solution as long as the graph has almost constant degree (such that $\rhomin \approx 1 \approx \rhomax$) and it is $K$ clusterable (i.e., the gap $\lambda_{K+1} - \lambda_{K}$ is large).

\begin{figure*}[t]
  \centering
  \vspace{-1mm}
  \hfill
  \subfloat[]{\includegraphics[width=0.35\textwidth]{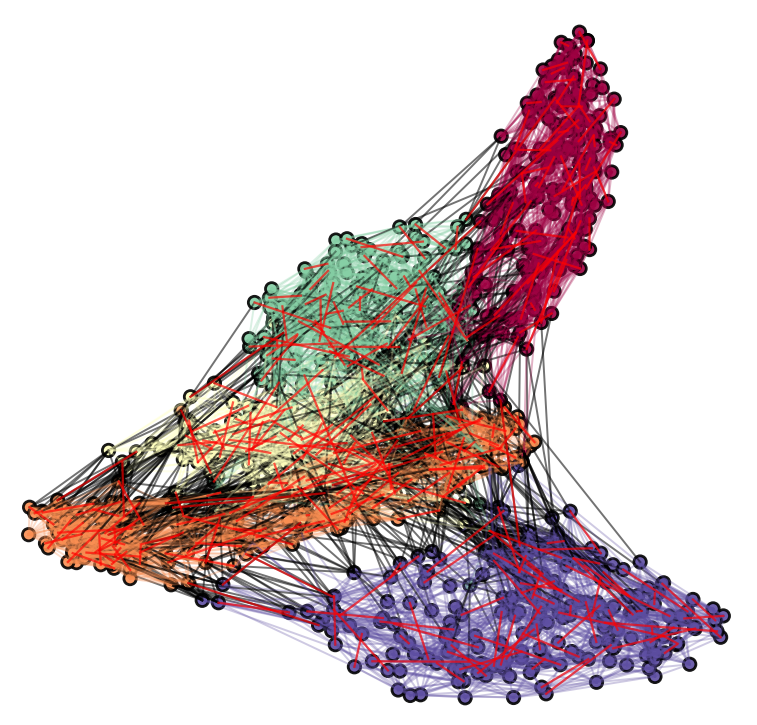}\label{fig:mnist5_graph}}
  ~
  \subfloat[]{\includegraphics[width=0.63\textwidth]{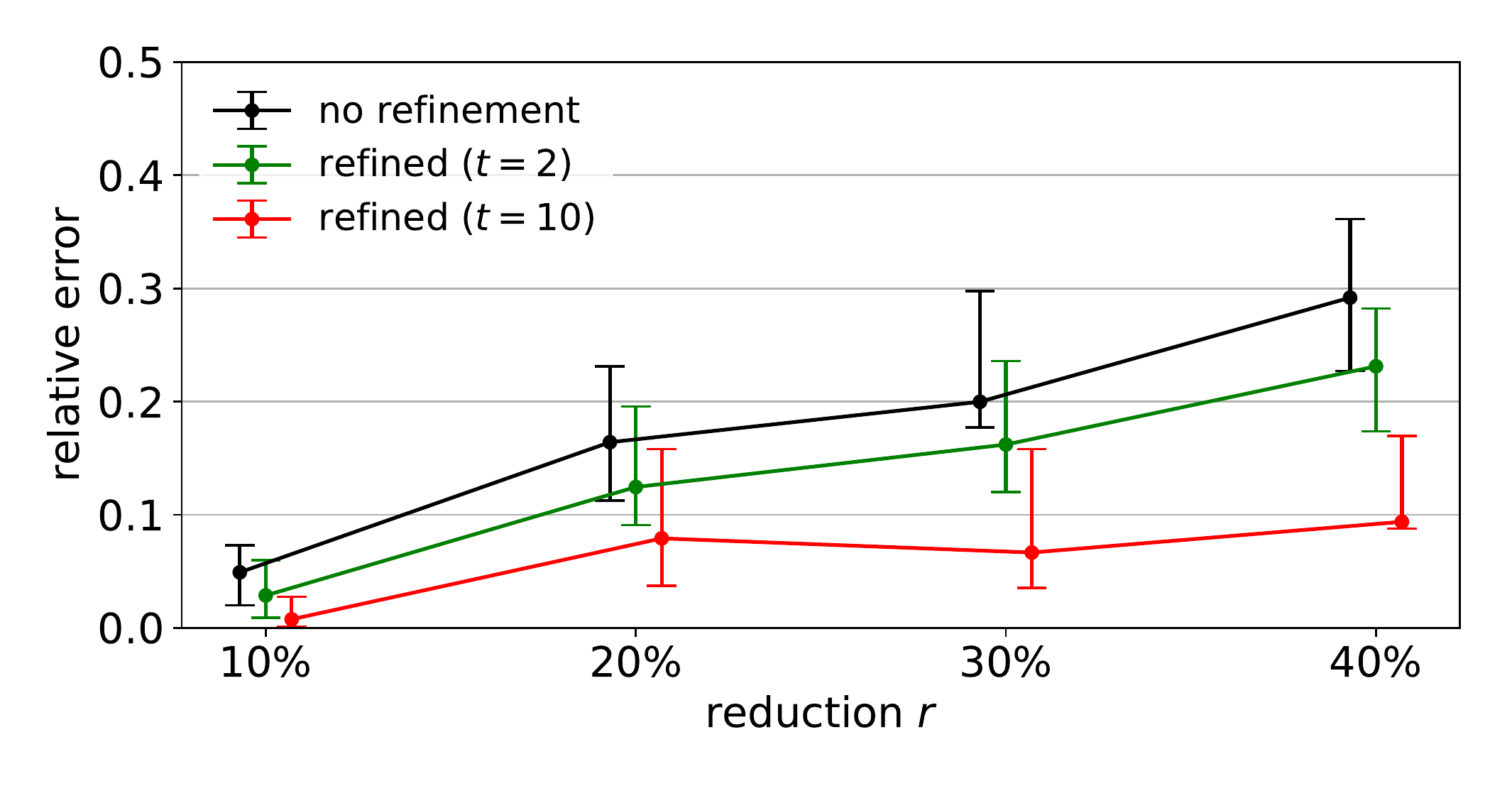}\label{fig:mnist5}}
  \hfill
 \vspace{2mm}
\caption{ (a) An instance of the clustering problem we considered. The graph is constructed from $N=1000$ images, each depicting a digit between 0 and 4 from the MNIST database. Contracted edges are shown in red. (b) The relative k-means error induced by coarsening as a function of dimensionality reduction $r$ (in percentage) for clustering 5 MNIST digits. The three lines correspond to the lifted eigenvectors with and without refinement. The errorbars span one standard deviation. A small horizontal offset has been inserted in order to diminish overlap. }
\end{figure*}

From the algebraic formulation of the k-means cost it follows that, when the number of clusters is $\kappa < K$ whereas the feature matrix remains $\Psi$, then $\kmeans{\kappa}{\Psi}{S^{*}} \geq K - \kappa$ (this is because $\Psi$ has exactly $K$ unit singular values, whereas the k-means clustering cannot do better than a rank $\kappa$ approximation of $\Psi$~\cite{ding2004k,boutsidis2015randomized}). Under the conditions of Corollary~\ref{corollary:spectral_clustering} and with the same probability: 
\begin{align}
		\left[\frac{{\kmeans{K}{\Psi}{\p{S}^{*}}}^{\frac{1}{2}} - {\kmeans{K}{\Psi}{S^{*}}}^{\frac{1}{2}}}{{\kmeans{\kappa}{\Psi}{S^*}}^{\sfrac{1}{2}}} \right]^{2} 
		\hspace{-2mm} \leq  \sum_{k = 2}^K  \frac{8 \epsilon r \lambda_k }{\delta_K \, (K - \kappa)},\notag 
\end{align}
which is a relaxed relative error guarantee.

Figure~\ref{fig:mnist5} depicts the growth of the actual relative error with $r$. The particular experiment corresponds to a clustering problem involving $N=1000$ images, each depicting a selected digit between 0 and 4 from the MNIST database (i.e., $K = 5$). We constructed a $12$-nearest neighbor similarity graph and repeated the experiment 10 times, each time using a different image set, selected uniformly at random. This setting produces a simple, but non-trivial, clustering problem featuring some overlaps between clusters (see Figure~\ref{fig:mnist5_graph}).

We observed that most remaining error occurred at coarsened vertices lying at cluster boundaries. This can be eliminated by only a few iterations of local smoothing. Though more advanced techniques might be preferable from a computational perspective, such as Chebychev or ARMA graph filters~\cite{shuman2011chebyshev,isufi2017autoregressive}, 
for illustration purposes, we here additionally perform $t = \{2,10\}$ steps of a simple power iteration scheme, yielding an $O(t M (K-1))$ overhead. The experiment confirms that most errors are removed after few iterations. 

\section{Discussion}

The main message of our work is the following: coarsening locally perturbs individual eigenvectors; however, if carefully constructed, it can leave well-separated principal eigenspaces relatively untouched. 

From perspective of manifold learning, where $G$ is a the discretization of the data manifold $\mathcal{M}$, our results can be interpreted as a statement of manifold approximation: the coarse information about the global geometry of $\mathcal{M}$ (as captured by eigenvectors) agrees with the information carried by the original discretization.  

The main limitation of our analysis is that the concentration estimates given by Theorem~\ref{theorem:contraction_similarity_general} are overly pessimistic. Proving tighter concentration is complicated by the dependency structure of the binomial random variables in the sum. In addition, our results are not currently applicable to multi-level coarsening, meaning that the maximum attainable dimensionality reduction ratio is below \hspace{-0.3mm}$\sfrac{1}{2}$. 
We are investigating ways to circumvent these issues. 


\appendix


\section{Appendix}
\subsection{Proof of Theorem~\ref{theorem:interlacing}}

\begin{proof}
The Courant-Fisher min-max theorem for $L$ reads  
\begin{align}
	\lambda_k = \min\limits_{\dimension{U} = k}  \max\limits_{x \in \spanning{U}} \left\{  \frac{x^\top L x}{x^\top x}    \right\}, \label{eq:courant_lambda_k}
\end{align}
whereas the same theorem for $\c{L}$ reads
\begin{align}
	\p{\lambda}_k = \min\limits_{\dimension{\c{U}} = k}   \max\limits_{\c{x} \in \spanning{\c{U}}} \ \left\{ \frac{\c{x}^\top \c{L} \c{x}}{\c{x}^\top \c{x}} \right\} 
	&= \min\limits_{\dimension{\c{U}} = k} \max\limits_{C x \in \spanning{\c{U}}} \left\{ \frac{ x^\top \Pi L \Pi x}{x^\top \Pi x}  \right\}  \notag \\
	&= \min\limits_{\dimension{U} = k}  \max\limits_{x \in \spanning{U}} \left\{  \frac{x^\top L x}{x^\top x} | \, x = \Pi x  \right\}, \notag  
\end{align}
where in the second equality we set $ \c{L} = C L C^\top$ and $\c{x} = C x $ and the third equality holds since $\Pi$ is a projection matrix (see Property~\ref{property:P}). Notice how, with the exception of the constraint that $x = \Pi x$, the final optimization problem is identical to the one for $\lambda_k$, given in~\eqref{eq:courant_lambda_k}. As such, the former's solution must be strictly larger (since it is a more constrained problem) and we have that $\p{\lambda}_k \geq \lambda_k$.
\end{proof}

\subsection{Proof of Theorem~\ref{theorem:contraction_similarity_general}}

We now proceed to derive the main statement of Theorem~\ref{theorem:contraction_similarity_general}.
Our approach will be to control $x_k^\top \p{L} x_k $ through its expectation.

\begin{lemma}
For any $k$ such that
$ \lambda_k \leq 0.5 \, \min_{ e_{ij} \in \mathcal{E}} \left\lbrace \frac{\deg_{i} + \deg_{j}}{2} + w_{ij} \right\rbrace$
the matrix $\c{L}$ produced by \alg\, abides to 
\begin{align}
 	\Prob{\lambda_k \leq x_k^\top \p{L} x_k \leq \lambda_k (1 + \epsilon)} \geq 1 - \frac{ \vartheta_k(T, \phi)}{4 \epsilon}, 
\end{align}
where 
\begin{align}
	\vartheta_k(T, \phi) = \max_{e_{ij} \in \mathcal{E}} \left\lbrace \Prob{ e_{ij} \in \mathcal{E}_F } \, \frac{\deg_i + \deg_j + 2 (w_{ij}-\lambda_k)}{w_{ij}} \right\rbrace.
\end{align}
\label{lemma:probability_estimate}
\end{lemma}
%
\begin{proof}
Denote by $\Pi^\bot$ the projection matrix defined such that $\Pi + \Pi^\bot = I$. We can then write
\begin{align}
	x_k^\top \p{L} x_k = x_k^\top \Pi L  \Pi x_k = x_k^\top (I - \Pi^\bot)  L (I - \Pi^\bot) x_k 
	&= x_k^\top L x_k - 2 x_k^\top L \Pi^\bot x_k + x_k^\top \Pi^\bot L \Pi^\bot x_k \notag \\
	&= \lambda_k - 2 \lambda_k x_k^\top \Pi^\bot x_k + x_k^\top \Pi^\bot L \Pi^\bot x_k 	\label{eq:theorem1_equality}
\end{align}
Let us now consider term $x_k^\top \Pi^\bot L \Pi^\bot x_k$, where for compactness we set $y = \Pi^\bot x_k$.
\begin{align}
	y^\top L y &= \sum_{e_{ij} \in \mathcal{E}} w_{ij} (y(i) - y(j))^2 = \underbrace{\sum_{e_{ij} \in \mathcal{E}_F} w_{ij} (y(i) - y(j))^2}_{T_1} + \underbrace{ \sum_{v_i \in \mathcal{V}_F} \sum_{v_j \notin \mathcal{V}_F} w_{ij} y(i)^2}_{T_2}. 
\end{align}
In the last step above, we exploited the fact that $y(i) = 0$ whenever $v_i \notin \mathcal{V}_F$. 

Since $\mathcal{E}_F$ is a matching of $\mathcal{E}$, any coarsening that occurs involves a merging of two adjacent vertices $v_i,v_j$ with $(\Pi x)(i) = (\Pi x)(j)$, implying that for every $e_{ij} \in \mathcal{E}_F$:
\begin{align}
	(y(i) - y(j))^2 = ((\Pi^\bot x_k)(i) + (\Pi x_k)(i) - (\Pi^\bot x_k)(j) - (\Pi x_k)(j))^2 = (x(i) - x(j))^2 \notag 
\end{align}
and therefore
%
\begin{align}
	T_1 &= \sum_{e_{ij} \in \mathcal{E}} b_{ij} \, w_{ij} \, (x_k(i) - x_k(j))^2,
	\label{eq:rec_2}
\end{align}
with $b_{ij}$ a Bernoulli random variable indicating whether $e_{ij} \in \mathcal{E}_F$. For $T_2$, notice that the terms in the sum correspond to boundary edges and, moreover, whenever $e_{ij} \in \mathcal{E}_F$ all vertices adjacent to $v_i$ and $v_j$ do not belong in $\mathcal{V}_F$. Another way to express $T_2$ therefore is 
\begin{align}
	T_2 &= 	\sum_{e_{ij} \in \mathcal{E}} b_{ij} \left( y(i)^2 \hspace{-4mm} \sum_{e_{i \ell} \in \mathcal{E}, e_{i \ell} \neq e_{ij}} \hspace{-4mm} w_{i \ell} + y(j)^2 \hspace{-4mm} \sum_{e_{j \ell} \in \mathcal{E}, e_{j \ell} \neq e_{ij}} \hspace{-4mm} w_{j \ell} \right)  \notag \\
	&= \sum_{e_{ij} \in \mathcal{E}} b_{ij} \left( \Big(x_k(i) - \frac{x_{k}(i) + x_k(j)}{2} \Big)^2 (\deg_i - w_{ij}) + \Big(x_k(j) - \frac{x_{k}(i) + x_k(j)}{2} \Big)^2 (\deg_j - w_{ij})\right)  \notag \\
	&= \sum_{e_{ij} \in \mathcal{E}} b_{ij} \, w_{ij} (x_k(i) - x_k(j))^2 \, \frac{\deg_i + \deg_j - 2 w_{ij}}{4 w_{ij}}.
	\label{eq:rec_3}
\end{align}
A similar result also holds for the remaining term $x_k^\top \Pi^\bot x_k = \| \Pi^\bot x_k\|_2^2$ of~\eqref{eq:theorem1_equality}: 
\begin{align}
	\| \Pi^\bot x_k\|_2^2 
	&= \sum_{e_{ij} \in \mathcal{E}} b_{ij} \left( \left( x_k(i)- \frac{x_k(i) + x_k(j)}{2} \right)^2 + \left( x_k(i)- \frac{x_k(i) + x_k(j)}{2} \right)^2 \right) \notag \\
	&= \sum_{e_{ij} \in \mathcal{E}} b_{ij} \, w_{ij} \left( x_k(i) - x_k(j)\right)^2 \frac{1}{ 2 w_{ij}}.
	\label{eq:rec_4}
\end{align}
If we substitute \eqref{eq:rec_2}, \eqref{eq:rec_3}, and~\eqref{eq:rec_4} into~\eqref{eq:theorem1_equality} we find that 
\begin{align}
	x_k^\top \p{L} x_k - \lambda_k  &= \sum_{e_{ij} \in \mathcal{E}} b_{ij} \, w_{ij} (x_k(i) - x_k(j))^2 \, \left( 1 + \frac{\deg_i + \deg_j - 2 w_{ij}}{4 w_{ij}} - \frac{\lambda_k}{w_{ij} }\right) \notag \\
	&= \frac{1}{4} \sum_{e_{ij} \in \mathcal{E}} b_{ij} \, w_{ij} (x_k(i) - x_k(j))^2 \, \left( \frac{\deg_i + \deg_j + 2 (w_{ij}-2\lambda_k)}{w_{ij}}\right) 
		\label{eq:randomized_positive}
\end{align}	
and furthermore
\begin{align}
	\E{x_k^\top \p{L} x_k} - \lambda_k 
	&= \frac{1}{4} \sum_{e_{ij} \in \mathcal{E}} \Prob{ e_{ij} \in \mathcal{E}_F } \left( \frac{\deg_i + \deg_j + 2 (w_{ij}-2\lambda_k)}{w_{ij}} \right) w_{ij} (x_k(i) - x_k(j))^2.
\end{align}
The expression above is always smaller than
\begin{align}
 \E{x_k^\top \p{L} x_k} - \lambda_k 
 &\leq \frac{\lambda_k}{4} \max_{e_{ij} \in \mathcal{E}} \left\lbrace \Prob{ e_{ij} \in \mathcal{E}_F } \, \frac{\deg_i + \deg_j + 2 (w_{ij}-2\lambda_k)}{w_{ij}} \right\rbrace = \frac{\lambda_k}{4} \, \vartheta_k(T, \phi),
 \label{eq:intermediate}
\end{align}
where $\vartheta_k(T, \phi)$ is a function of the sampling probabilities, the eigenvalue $\lambda_k$, and the degree distribution of $G$.
Noticing that~\eqref{eq:randomized_positive} is a non-negative random variable whenever $\lambda_k \leq 0.5 \min_{e_{ij} \in \mathcal{E}} \frac{\deg_{i}	+ \deg_{j}}{2} + w_{ij}/2$ (the condition is equivalent to $\deg_i + \deg_j + 2 (w_{ij}-2\lambda_k)> 0$ implying that $x_k^\top \p{L} x_k - \lambda_k$ is a sum of non-negative terms) and using Markov's inequality, we find that 
\begin{align}
	\Prob{x_k^\top \p{L} x_k \geq \lambda_k (1 + \epsilon)} = \Prob{\frac{x_k^\top \p{L} x_k - \lambda_k}{\lambda_k} \geq \epsilon} \leq \frac{\E{x_k^\top \p{L} x_k} - \lambda_k}{\epsilon \lambda_k} \leq \frac{ \vartheta_k(T, \phi)}{4 \epsilon},
\end{align}
which gives the desired probability bound.
\end{proof}

The RSS constant therefore depends on the probability that each edge $e_{ij}$ is contracted. This is given by:
%
\begin{lemma}
	At the termination of \alg, each edge $e_{ij}$ of $\mathcal{E}$ can be found in $\mathcal{E}_F$ with probability
	\begin{align}
	p_{ij} \, \frac{1- e^{-T P_{ij}}}{P_{ij}} \leq \Prob{ e_{ij} \in \mathcal{E}_F } = \Prob{ b_{ij}  = 1} \leq \, p_{ij} \frac{1- e^{-T P_{ij}}}{1 - e^{-P_{ij}}} 
	\end{align}
	where $p_{ij} = \phi_{ij} / \Phi$ and $P_{ij} = \sum_{ e_{pq} \in \mathcal{N}_{ij} } p_{pq}$. 
	\label{lemma:probability}
\end{lemma}
%
\begin{proof}
The event $X_{ij}(t)$ that edge $e_{ij}$ is still in the candidate set $\mathcal{C}$ at the end of the $t$-th iteration is 
\begin{align}
	\Prob{ X_{ij}(t)} &= \Prob{ X_{ij}(t-1) \cap \left\lbrace e_{ij} \text{ is not selected at } t\right\rbrace} \notag \\
	&= \Prob{ X_{ij}(t-1)} \prod_{pq \in \mathcal{N}_{ij}} (1-p_{pq}) = \prod_{\tau = 1}^{t} \left(\prod_{pq \in \mathcal{N}_{ij}} (1-p_{pq}) \right)= a_{ij}^t.
\end{align} 
Therefore, the probability that $e_{ij}$ is selected after $T$ iterations can be written as 
\begin{align}
	\Prob{ e_{ij} \in \mathcal{E}_F } &= \sum_{t = 1}^T \Prob{ e_{ij} \text{ is selected at } t} \notag \\
	&= \sum_{t = 1}^T p_{ij} \Prob{ X_{ij}(t-1)} \notag \\
	&= p_{ij} \sum_{t = 0}^{T-1} a_{ij}^t = p_{ij} \, \frac{1 - a_{ij}^{T}}{1 - a_{ij}}. 
	\label{eq:probability_edge}
\end{align}
According to the Weierstrass product inequality
\begin{align}
a_{ij} = \prod_{ e_{pq} \in \mathcal{N}_{ij} } (1-p_{pq}) 
&\geq 1 - \sum_{ e_{pq} \in \mathcal{N}_{ij} } p_{pq}  
\end{align}
and since the function $f(x) = (1 - x^T)/(1-x)$ is monotonically increasing in $[0,1]$ and setting $P_{ij} = \sum_{ e_{pq} \in \mathcal{N}_{ij} } p_{pq} $ we have that
$$ \frac{1 - a_{ij}^{T}}{1 - a_{ij}} \geq \frac{1 - (1-P_{ij})^T }{P_{ij}} = \frac{1 - e^{\log{(1-P_{ij})}T} }{P_{ij}} \geq \frac{1 - e^{-T P_{ij} }}{P_{ij}}, $$ 
%
where the last step takes advantage of the series expansion $\log{\left( 1-p \right)} = -\sum_{i = 1}^\infty p^{i}/i \leq - p$.
Similarly, for the upper bound 
\begin{align}
a_{ij} 
= \prod_{ e_{pq} \in \mathcal{N}_{ij} } (1-p_{pq}) 
&= e^{ \log{\left(  \prod_{ e_{pq} \in \mathcal{N}_{ij} } (1-p_{pq}) \right)} } = e^{ \sum_{{ e_{pq} \in \mathcal{N}_{ij} }} \log{\left( 1-p_{pq} \right)} } \leq e^{ -\sum_{{ e_{pq} \in \mathcal{N}_{ij} }} p_{pq} } = e^{-P_{ij}} 
\end{align}
and therefore $ \frac{1 - a_{ij}^{T}}{1 - a_{ij}} \leq \frac{1- e^{-T\,P_{ij}}}{1 - e^{-P_{ij}}}, $ as claimed.
\end{proof}

Based on Lemma~\ref{lemma:probability}, the expression of $\vartheta_k(T, \phi)$ is
\begin{align}
	\vartheta_k(T, \phi) 
	&\leq \max_{e_{ij} \in \mathcal{E}} \left\lbrace p_{ij} \frac{1- e^{-T P_{ij}}}{1 - e^{-P_{ij}}} \, \frac{\deg_i + \deg_j + 2 (w_{ij}-2\lambda_k)}{w_{ij}} \right\rbrace \notag \\
	&\leq \max_{e_{ij} \in \mathcal{E}} \left\lbrace P_{ij} \, \frac{1- e^{-T P_{ij}}}{1 - e^{-P_{ij}}} \right\rbrace \max_{e_{ij} \in \mathcal{E}} \left\lbrace \frac{p_{ij}}{P_{ij}} \, \frac{\deg_i + \deg_j + 2 (w_{ij}-2\lambda_k)}{w_{ij}} \right\rbrace. 
\end{align}
The function $f(P_{ij}) = P_{ij} \, \frac{1 - e^{-T P_{ij}}}{1 - e^{- P_{ij}}}$ has a positive derivative in the domain of interest and thus it attains its maximum at $\Pmax$ when $P_{ij}$ is also maximized.  
Setting $c_1 = N \Pmax$ and after straightforward algebraic manipulation, we find:    
\begin{align}
	\vartheta_k(T, \phi) 
	&\leq  \Pmax \, \frac{1- e^{- c_1 T /N}}{1 - e^{-\Pmax}} \max_{e_{ij} \in \mathcal{E}} \left\lbrace \frac{p_{ij}}{P_{ij}} \, \frac{\deg_i + \deg_j + 2 (w_{ij}-2\lambda_k)}{w_{ij}} \right\rbrace \notag \\
	&= \Pmax \, \frac{1- e^{- c_1 T /N}}{1 - e^{-\Pmax}} \max_{e_{ij} \in \mathcal{E}} \left\lbrace \frac{\phi_{ij}}{ \sum_{e_{pq} \in \mathcal{N}_{ij}} \phi_{pq}} \, \left( \frac{ \sum_{e_{pq} \in \mathcal{N}_{ij}} w_{pq} }{w_{ij}} + 3 -\frac{4\lambda_k}{w_{ij}} \right) \right\rbrace.
\end{align}

For any potential function and graph such that $\Pmax = O(1/N)$, at the limit $c_2 = \frac{\Pmax}{1 - e^{-\Pmax}} {\rightarrow} 1$ and the above expression reaches
\begin{align}
	\lim_{N \rightarrow \infty} \vartheta_k(T, \phi) 
	\leq (1- e^{- c_1 T/N}) \, \max_{e_{ij} \in \mathcal{E}} \left\lbrace \frac{\phi_{ij}}{ \sum_{e_{pq} \in \mathcal{N}_{ij}} \phi_{pq}} \, \left( \frac{ \sum_{e_{pq} \in \mathcal{N}_{ij}} w_{pq} }{w_{ij}} + 3 -\frac{4\lambda_k}{w_{ij}} \right) \right\rbrace.
\end{align}
The final probability estimate is achieved by using Lemma~\ref{lemma:probability_estimate} along with the derived bound on $\vartheta_k(T, \phi)$.

\subsection{Proof of Theorem~\ref{theorem:eigenvalue}}

We adopt a variational approach and reason that, since  
\begin{align}
\p{\lambda}_k &= \min\limits_{U}  \max\limits_{x} \left\{  \frac{x^\top L x}{x^\top x}, \, x \in {U} \text{ and } x \neq 0\, | \, \dimension{U} = k  \, | \, x = \Pi x  \right\},  
\end{align}
for any matrix $Z$ the following inequality holds
\begin{align}
\p{\lambda}_k \leq \max\limits_{x} \left\{  \frac{x^\top L x}{x^\top x}  \, | \, x \in \spanning{Z} \text{ and } x \neq 0\, \right\}  
\label{eq:bbb1}
\end{align}
as long as the columnspace of $Z$ is of dimension $k$ and does not intersect with the nullspace of $\Pi$. 

Write $\p{X}_{k-1}$ to denote the $n \times (k-1)$ matrix with the $k-1$ first eigenvectors of $\c{L}$ and further set $Y_{k-1} = C^\top \p{X}_{k-1}$.
We will consider the $N\times k$ matrix $Z$ with 
\begin{align}
Z(:,i) = 
\begin{cases}
C^\top \p{x}_i       & \quad \text{if } i < k\\
z  &  \quad \text{if } i = k,
\end{cases}  
\quad \text{where} \quad 
z = \Pi (I - Y_{k-1} Y_{k-1}^\top) x_k.
\end{align}
It can be confirmed that $Z$'s columnspace meets the necessary requirements. 
Now, we can express any $x \in \spanning{Z}$ as $x = Y_{k-1} a + b z = \Pi (Y_{k-1} a + b z)$ with $ \|a\|^2 + b^2\|z\|^2 = 1$ and therefore
\begin{align}
x^\top L x 
&=  (a^\top Y_{k-1}^\top + b z^\top) \Pi {L} \Pi ( Y_{k-1} a + b z) \notag \\
&=  (a^\top Y_{k-1}^\top + b z^\top) \p{L} ( Y_{k-1} a + b z) \notag \\
&= a^\top Y_{k-1}^\top \p{L} Y_{k-1} a + b^2 \, z^\top \p{L} z + 2 b \, z^\top \p{L} Y_{k-1} a \notag \\
&= a^\top Y_{k-1}^\top \p{L} Y_{k-1} a + b^2 \, z^\top \p{L} z,
\end{align} 
where in the last step we exploited the fact that, by construction, $z$ does not lie in the span of $\p{X}_{k-1}$ (matrix $\p{L}$ does not rotate its own eigenvectors).
Since $Y_{k-1} a \in \spanning{\p{X}_{k-1}}$, the first term in the equation above in bounded by $\p{\lambda}_{k-1}$ and the equality is attained only when $a(k-1) = 1$ (in which case $b$ must be zero).
By the variational argument however, we are certain that the upper bound in \eqref{eq:bbb1} has to be at least as large as $\p{\lambda}_{k-1}$, implying that
\begin{align}
\p{\lambda}_k \leq \max\left\lbrace \p{\lambda}_{k-1}, \frac{z^\top L z}{z^\top z} \right\rbrace
\end{align} 
with the two cases corresponding to the choices $a(k-1) = 1$ and $b = 1$, respectively.
In addition, we have that 
\begin{align}
z^\top L z &= x_k^\top (I - Y_{k-1} Y_{k-1}^\top) \Pi L \Pi (I - Y_{k-1} Y_{k-1}^\top) x_k = \sum_{i \geq k } \p{\lambda}_i \, ( \p{x}_i^\top C x_k)^2 
\end{align} 
and $\|z\|^2 = \|\Pi (I - Y_{k-1} Y_{k-1}^\top) x_k\|^2 = \sum_{i \geq k} ( \p{x}_i^\top C x_k)^2$,
meaning that 
\begin{align}
\frac{z^\top \p{L} z }{z^\top z}
&= \frac{\sum_{i \geq k } \p{\lambda}_i \, ( \p{x}_i^\top C x_k)^2}{\sum_{i \geq k} ( \p{x}_i^\top C x_k)^2}  
\leq \frac{x_k^\top \p{L} x_k }{\sum_{i \geq k} ( \p{x}_i^\top C x_k)^2} 
\end{align} 
and therefore the relation $\p{\lambda}_k \leq \max\left\lbrace \p{\lambda}_{k-1},  (1 + \epsilon_k) \frac{\lambda_k}{\sum_{i \geq k} \theta_{ki}}\right\rbrace$ holds whenever $k \leq K$.

\subsection{Proof of Theorem~\ref{theorem:sintheta}}

\begin{proof}
Li's Lemma~\cite{li1994relative} allows to express $\vartheta_k$ based on the squared inner products $(\p{x}_j^\top C x_i)^2$ of the eigenvectors $x_i$ of the Laplacian ${L}$ and the lifted eigenvectors $C^\top \p{x}_j$ of the coarsened Laplacian $\c{L}$.
\begin{align}
	\vartheta_k = \norm{ \sintheta{X_k}{C^\top \p{X}_k} }_F^2 
	&= \norm{\p{X}_{k^\bot}^\top C X_k}_F^2 
	= \sum_{i \leq k} \sum_{j > k} (\p{x}_j^\top C x_i)^2
\label{eq:sintheta_0}
\end{align}
Moreover, the summed RSS inequalities for each $i \leq k$ give:
\begin{align}
	\sum_{i \leq k} (1 + \epsilon_i)\lambda_i \geq \sum_{i \leq k}  x_i^\top \p{L} x_i 
	&=  \sum_{i \leq k} \sum_{j = 1}^n \p{\lambda}_j (\p{x}_j^\top C x_i)^2 =   \sum_{j \leq k} \p{\lambda}_j \sum_{i \leq k} (\p{x}_j^\top C x_i)^2 + \sum_{j > k} \p{\lambda}_j \sum_{i \leq k} (\p{x}_j^\top C x_i)^2.
\label{eq:sintheta_1}
\end{align}
To continue, we use the equality 
\begin{align} 
	\sum_{2 \leq j \leq k} \sum_{ i \leq k} (\p{x}_j^\top C x_i)^2 = \sum_{2\leq i \leq k} \left(\| \Pi x_i \|_2^2 - \sum_{j > k} (\p{x}_j^\top C x_i)^2 \right)
\end{align}
based on which 
\begin{align}
	\p{\lambda}_{k+1} \sum_{j > k} \sum_{i \leq k} (\p{x}_j^\top C x_i)^2 + \p{\lambda}_2 \sum_{2 \leq i \leq k} \left(\| \Pi x_i \|_2^2 - \sum_{j > k} (\p{x}_j^\top C x_i)^2 \right)
	\leq \sum_{i \leq k} (1 + \epsilon_i)\lambda_i = \sum_{2 \leq i \leq k} (1 + \epsilon_i)\lambda_i. 
\end{align}
Our first sin$\Theta$ bound is obtained by using the inequality $\lambda_2 \leq \p{\lambda}_2$ and re-arranging the terms:
\begin{align}
	\norm{ \sintheta{X_k}{C^\top \p{X}_k} }_F^2 \leq \sum\limits_{2\leq i \leq k}  \frac{ (1 + \epsilon_i)\lambda_i - \lambda_2 \| \Pi x_i \|_2^2 }{\p{\lambda}_{k+1} - \lambda_2} 
	\label{eq:sintheta_2}
\end{align}
For the second bound, we instead perform the following manipulation 
\begin{align}
	\sum_{j \leq k} \p{\lambda}_j \sum_{i \leq k} (\p{x}_j^\top C x_i)^2 \geq \sum_{j \leq k} \lambda_j \sum_{i \leq k} (\p{x}_j^\top C x_i)^2  &=  \sum_{j \leq k} \lambda_j \left( 1 - \sum_{i > k} (\p{x}_j^\top C x_i)^2 \right) \notag \\
	&\geq \sum_{j \leq k} \lambda_j - \lambda_k  \sum_{i \leq k} \left( \| \Pi^\bot x_i\|_2^2 + \sum_{j \geq k} (\p{x}_j^\top C x_i)^2 \right),
\end{align}
which together with~\eqref{eq:sintheta_0} and~\eqref{eq:sintheta_1} results to
\begin{align}
	\norm{ \sintheta{X_k}{C^\top \p{X}_k} }_F^2 
	\leq \sum\limits_{i \leq k}  \frac{  (1 + \epsilon_i)\lambda_i - \lambda_i  + \lambda_k \| \Pi^\bot x_i \|_2^2 }{\p{\lambda}_{k+1} - \lambda_{k}} 
	= \sum\limits_{2\leq i \leq k} \frac{ \epsilon_i\lambda_i + \lambda_k \| \Pi^\bot x_i \|_2^2 }{\p{\lambda}_{k+1} - \lambda_{k}}.
	\label{eq:sintheta_3}
\end{align}
The final bound is obtained as the minimum of \eqref{eq:sintheta_2} and~\eqref{eq:sintheta_3}.
\end{proof}

\subsection{Proof of Corollary~\ref{corollary:spectral_clustering}}

\begin{proof}
The proof follows a known argument in the analysis of spectral clustering first proposed by Boutsidis~\cite{boutsidis2015spectral} and later adapted by Martin et al.~\cite{martin2017fast}. In particular, these works proved that:
\begin{align}
	\kmeans{K}{\Psi}{\p{S}^*}^{\sfrac{1}{2}} &\leq \kmeans{K}{\Psi}{S^*}^{\sfrac{1}{2}} + 2 \, \gamma_K,
	\label{eq:kmeans_1}
\end{align} 
with $\gamma_K = \| \Psi - \p{\Psi} Q\|_F = \| X_K - C^\top \p{X}_K Q\|_F$ and $Q$ being some unitary matrix of appropriate dimensions.
However, as demonstrated by Yu and coauthors~\cite{yu2014useful}, it is always possible to find a unitary matrix $Q$ such that
\begin{align}
	\gamma_K^2 = \norm{ X_K - C^\top \p{X}_K Q}_F^2 \leq 2 \norm{\sintheta{X_K}{C^\top \p{X}_K} }_F^2 \leq 2 \sum_{k=2}^{K} \frac{ \epsilon_k\lambda_k + \lambda_K \| \Pi^\bot x_k \|_2^2 }{\delta_K}
\end{align}
where the last inequality follows from Theorem~\ref{theorem:sintheta} and $\p{\lambda}_{K+1} \geq \lambda_{K+1}$. At this point, we could opt to take a union bound with respect to the events $\{\epsilon_k \geq \epsilon\}$ and $\{\| \Pi^\bot x_k\|_2^2 \geq \epsilon \}$ using the results of Section~\ref{sec:algorithm}. A more careful analysis however follows the steps of the proof of Theorem~\ref{theorem:contraction_similarity_general} simultaneously for all terms: 
\begin{align}
\sum_{k=2}^{K} \E{\epsilon_k}\lambda_k + \lambda_K \E{\| \Pi^\bot x_k \|_2^2} &= \sum_{k=2}^{K}  \sum_{e_{ij} \in \mathcal{E}} \Prob{e_{ij} \in \mathcal{E}_F} w_{ij} (x_k(i) - x_k(j))^2 \left[ \frac{d_i + d_j + 2 w_{ij} + 2\lambda_K - 4 \lambda_k}{4 w_{ij}} \right] \notag \\
&\hspace{-3cm}\leq \sum_{k=2}^{K}  \lambda_k \max_{e_{ij} \in \mathcal{E}} \left\lbrace \Prob{e_{ij} \in \mathcal{E}_F} \left[ \frac{d_i + d_j + 2 w_{ij} + 2\lambda_K - 4 \lambda_k}{4 w_{ij}} \right] \right\rbrace \notag \\
&\hspace{-3cm}\leq \sum_{k=2}^{K}  \lambda_k \, \Pmax \frac{1 - e^{-T\Pmax}}{1 - e^{-\Pmax}} \, \max_{e_{ij} \in \mathcal{E}} \left\lbrace \frac{\phi_{ij}}{ \sum_{e_{pq} \in \mathcal{N}_{ij}} \phi_{pq}} \, \left( \frac{ \sum_{e_{pq} \in \mathcal{N}_{ij}} w_{pq} }{w_{ij}} + 3 + \frac{2 \lambda_K - 4\lambda_k}{w_{ij}} \right) \right\rbrace \notag  \\ 
&\hspace{-3cm}= c_2 \frac{1- e^{- c_1 T /N} }{4}\sum_{k=2}^{K}  \lambda_k \, \max_{e_{ij} \in \mathcal{E}} \left\lbrace \frac{\phi_{ij}}{ \sum_{e_{pq} \in \mathcal{N}_{ij}} \phi_{pq}} \, \left( \frac{ \sum_{e_{pq} \in \mathcal{N}_{ij}} w_{pq} }{w_{ij}} + 3 + \frac{2 \lambda_K - 4\lambda_k}{w_{ij}} \right) \right\rbrace, 
\end{align}
where as before $c_1 = N \Pmax$ and $c_2 = \Pmax / (1 - e^{-\Pmax})$. Assuming further that a heavy-edge potential is used, $N$ is sufficiently large, and $G$ has bounded degree such that $c_1 = 4 \rhomax = O(1)$, the above simplifies to 
\begin{align}
\E{\gamma_K^2} 
&\leq \frac{1- e^{- 4 \rhomax T /N} }{2 \, \delta_K} \sum_{k=2}^{K}  \lambda_k \, \left( 1 +  \max_{e_{ij} \in \mathcal{E}} \left\lbrace \frac{3 w_{ij} + 2 \lambda_K - 4\lambda_k}{\sum_{e_{pq} \in \mathcal{N}_{ij}} w_{pq}} \right\rbrace \right) \notag \\
&\leq \frac{1- e^{- 4 \rhomax T /N} }{2 \, \delta_K} \sum_{k=2}^{K}  \lambda_k \, \left( 1 +  \max_{e_{ij} \in \mathcal{E}} \left\lbrace \frac{6 + 4 \lambda_K - 8\lambda_k}{ \davg \rhomin} \right\rbrace \right).
\end{align}
The last inequality used the relation $\min_{e_{ij}} \, \sum_{e_{pq} \in \mathcal{N}_{ij}} w_{pq} = \rhomin \davg / 2$ and the fact that $w_{ij} \leq 1.$
Setting $c_3 = \frac{\sum_{k=2}^{K} \lambda_k^2}{\sum_{k=2}^{K} \lambda_k}$, gives
\begin{align}
\E{\gamma_K^2} 
&\leq \frac{1- e^{- 4 \rhomax T /N} }{2 \,\delta_K} \left(\sum_{k=2}^{K}  \lambda_k\right)  \, \left( 1 + \frac{6 + 4 \lambda_K - 8\, c_3}{\davg \rhomin} \right).
\end{align}
From Markov's inequality, then 
\begin{align}
	\Prob{\left[\kmeans{K}{\Psi}{\p{S}^*}^{\sfrac{1}{2}} - \kmeans{K}{\Psi}{S^*}^{\sfrac{1}{2}}\right]^2 \geq  \epsilon   \sum_{k=2}^{K}  \frac{2 \lambda_k (1- e^{- 4 \rhomax T /N})}{\delta_K}  } \leq \frac{1}{\epsilon} \left( 1 + \frac{6 + 4 \lambda_K - 8 \, c_3}{\davg \rhomin} \right).
\end{align} 
The final result follows by the inequality $1- e^{- 4 \rhomax T /N} \leq 4 r \rhomax$ (see~\eqref{eq:ratio_inequality}). 
\end{proof}



\bibliographystyle{plain}
\bibliography{references}

\begin{thebibliography}{10}

\bibitem{belkin2003laplacian}
Mikhail Belkin and Partha Niyogi.
\newblock Laplacian eigenmaps for dimensionality reduction and data
  representation.
\newblock {\em Neural computation}, 15(6):1373--1396, 2003.

\bibitem{boutsidis2015spectral}
Christos Boutsidis, Prabhanjan Kambadur, and Alex Gittens.
\newblock Spectral clustering via the power method-provably.
\newblock In {\em International Conference on Machine Learning}, pages 40--48,
  2015.

\bibitem{boutsidis2015randomized}
Christos Boutsidis, Anastasios Zouzias, Michael~W Mahoney, and Petros Drineas.
\newblock Randomized dimensionality reduction for k-means clustering.
\newblock {\em IEEE Transactions on Information Theory}, 61(2):1045--1062,
  2015.

\bibitem{7974879}
M.~M. Bronstein, J.~Bruna, Y.~LeCun, A.~Szlam, and P.~Vandergheynst.
\newblock Geometric deep learning: Going beyond euclidean data.
\newblock {\em IEEE Signal Processing Magazine}, 34(4):18--42, July 2017.

\bibitem{bruna2014spectral}
Joan Bruna, Wojciech Zaremba, Arthur Szlam, and Yann Lecun.
\newblock Spectral networks and locally connected networks on graphs.
\newblock In {\em International Conference on Learning Representations
  (ICLR2014), CBLS, April 2014}, 2014.

\bibitem{chen2004interlacing}
Guantao Chen, George Davis, Frank Hall, Zhongshan Li, Kinnari Patel, and
  Michael Stewart.
\newblock An interlacing result on normalized laplacians.
\newblock {\em SIAM Journal on Discrete Mathematics}, 18(2):353--361, 2004.

\bibitem{chung1997spectral}
Fan~RK Chung.
\newblock {\em Spectral graph theory}.
\newblock Number~92. American Mathematical Soc., 1997.

\bibitem{colley2017algebraic}
Charles Colley, Junyuan Lin, Xiaozhe Hu, and Shuchin Aeron.
\newblock Algebraic multigrid for least squares problems on graphs with
  applications to hodgerank.
\newblock In {\em Parallel and Distributed Processing Symposium Workshops
  (IPDPSW), 2017 IEEE International}, pages 627--636. IEEE, 2017.

\bibitem{davis1970rotation}
Chandler Davis and William~Morton Kahan.
\newblock The rotation of eigenvectors by a perturbation. iii.
\newblock {\em SIAM Journal on Numerical Analysis}, 7(1):1--46, 1970.

\bibitem{defferrard2016convolutional}
Micha{\"e}l Defferrard, Xavier Bresson, and Pierre Vandergheynst.
\newblock Convolutional neural networks on graphs with fast localized spectral
  filtering.
\newblock In {\em Advances in Neural Information Processing Systems}, pages
  3844--3852, 2016.

\bibitem{dhillon2007weighted}
Inderjit~S Dhillon, Yuqiang Guan, and Brian Kulis.
\newblock Weighted graph cuts without eigenvectors a multilevel approach.
\newblock {\em IEEE transactions on pattern analysis and machine intelligence},
  29(11), 2007.

\bibitem{ding2004k}
Chris Ding and Xiaofeng He.
\newblock K-means clustering via principal component analysis.
\newblock In {\em Proceedings of the twenty-first international conference on
  Machine learning}, page~29. ACM, 2004.

\bibitem{gandhi2016improvement}
Shivam Gandhi.
\newblock Improvement of the cascadic multigrid algorithm with a gauss seidel
  smoother to efficiently compute the fiedler vector of a graph laplacian.
\newblock {\em arXiv preprint arXiv:1602.04386}, 2016.

\bibitem{gavish2010multiscale}
Matan Gavish, Boaz Nadler, and Ronald~R Coifman.
\newblock Multiscale wavelets on trees, graphs and high dimensional data:
  Theory and applications to semi supervised learning.
\newblock In {\em ICML}, pages 367--374, 2010.

\bibitem{hendrickson1995multi}
Bruce Hendrickson and Robert~W Leland.
\newblock A multi-level algorithm for partitioning graphs.
\newblock {\em SC}, 95(28):1--14, 1995.

\bibitem{hirani2015graph}
Anil~N Hirani, Kaushik Kalyanaraman, and Seth Watts.
\newblock Graph laplacians and least squares on graphs.
\newblock In {\em Parallel and Distributed Processing Symposium Workshop
  (IPDPSW), 2015 IEEE International}, pages 812--821. IEEE, 2015.

\bibitem{isufi2017autoregressive}
Elvin Isufi, Andreas Loukas, Andrea Simonetto, and Geert Leus.
\newblock Autoregressive moving average graph filtering.
\newblock {\em IEEE Transactions on Signal Processing}, 65(2):274--288, 2017.

\bibitem{karypis1998fast}
George Karypis and Vipin Kumar.
\newblock A fast and high quality multilevel scheme for partitioning irregular
  graphs.
\newblock {\em SIAM Journal on scientific Computing}, 20(1):359--392, 1998.

\bibitem{karypis1998multilevelk}
George Karypis and Vipin Kumar.
\newblock Multilevelk-way partitioning scheme for irregular graphs.
\newblock {\em Journal of Parallel and Distributed computing}, 48(1):96--129,
  1998.

\bibitem{koren2002fast}
David Harel~Yehuda Koren.
\newblock A fast multi-scale method for drawing large graphs.
\newblock {\em Journal of graph algorithms and applications}, 6(3):179--202,
  2002.

\bibitem{koutis2011combinatorial}
Ioannis Koutis, Gary~L Miller, and David Tolliver.
\newblock Combinatorial preconditioners and multilevel solvers for problems in
  computer vision and image processing.
\newblock {\em Computer Vision and Image Understanding}, 115(12):1638--1646,
  2011.

\bibitem{kushnir2006fast}
Dan Kushnir, Meirav Galun, and Achi Brandt.
\newblock Fast multiscale clustering and manifold identification.
\newblock {\em Pattern Recognition}, 39(10):1876--1891, 2006.

\bibitem{lafon2006diffusion}
Stephane Lafon and Ann~B Lee.
\newblock Diffusion maps and coarse-graining: A unified framework for
  dimensionality reduction, graph partitioning, and data set parameterization.
\newblock {\em IEEE transactions on pattern analysis and machine intelligence},
  28(9):1393--1403, 2006.

\bibitem{li1994relative}
Ren-Cang Li.
\newblock Relative perturbation theory:(ii) eigenspace variations.
\newblock Technical report, 1994.

\bibitem{livne2012lean}
Oren~E Livne and Achi Brandt.
\newblock Lean algebraic multigrid (lamg): Fast graph laplacian linear solver.
\newblock {\em SIAM Journal on Scientific Computing}, 34(4):B499--B522, 2012.

\bibitem{martin2017fast}
Lionel Martin, Andreas Loukas, and Pierre Vandergheynst.
\newblock Fast approximate spectral clustering for dynamic networks.
\newblock {\em arXiv preprint arXiv:1706.03591}, 2017.

\bibitem{muja2014scalable}
Marius Muja and David~G Lowe.
\newblock Scalable nearest neighbor algorithms for high dimensional data.
\newblock {\em IEEE Transactions on Pattern Analysis and Machine Intelligence},
  36(11):2227--2240, 2014.

\bibitem{nr-aaai15}
Ryan~A. Rossi and Nesreen~K. Ahmed.
\newblock The network data repository with interactive graph analytics and
  visualization.
\newblock In {\em Proceedings of the Twenty-Ninth AAAI Conference on Artificial
  Intelligence}, 2015.

\bibitem{shuman2016multiscale}
David~I Shuman, Mohammad~Javad Faraji, and Pierre Vandergheynst.
\newblock A multiscale pyramid transform for graph signals.
\newblock {\em IEEE Transactions on Signal Processing}, 64(8):2119--2134, 2016.

\bibitem{shuman2011chebyshev}
David~I Shuman, Pierre Vandergheynst, and Pascal Frossard.
\newblock Chebyshev polynomial approximation for distributed signal processing.
\newblock In {\em Distributed Computing in Sensor Systems and Workshops
  (DCOSS), 2011 International Conference on}, pages 1--8. IEEE, 2011.

\bibitem{spielman2011graph}
Daniel~A Spielman and Nikhil Srivastava.
\newblock Graph sparsification by effective resistances.
\newblock {\em SIAM Journal on Computing}, 40(6):1913--1926, 2011.

\bibitem{spielman2011spectral}
Daniel~A Spielman and Shang-Hua Teng.
\newblock Spectral sparsification of graphs.
\newblock {\em SIAM Journal on Computing}, 40(4):981--1025, 2011.

\bibitem{stewart1990matrix}
Gilbert~W Stewart.
\newblock Matrix perturbation theory.
\newblock 1990.

\bibitem{tremblay2016compressive}
Nicolas Tremblay, Gilles Puy, R{\'e}mi Gribonval, and Pierre Vandergheynst.
\newblock Compressive spectral clustering.
\newblock In {\em International Conference on Machine Learning}, pages
  1002--1011, 2016.

\bibitem{turk1994zippered}
Greg Turk and Marc Levoy.
\newblock Zippered polygon meshes from range images.
\newblock In {\em Proceedings of the 21st annual conference on Computer
  graphics and interactive techniques}, pages 311--318. ACM, 1994.

\bibitem{urschel2014cascadic}
John~C Urschel, Xiaozhe Hu, Jinchao Xu, and Ludmil~T Zikatanov.
\newblock A cascadic multigrid algorithm for computing the fiedler vector of
  graph laplacians.
\newblock {\em arXiv preprint arXiv:1412.0565}, 2014.

\bibitem{von2007tutorial}
Ulrike Von~Luxburg.
\newblock A tutorial on spectral clustering.
\newblock {\em Statistics and computing}, 17(4):395--416, 2007.

\bibitem{walshaw2006multilevel}
Chris Walshaw.
\newblock A multilevel algorithm for force-directed graph-drawing.
\newblock {\em Journal of Graph Algorithms and Applications}, 7(3):253--285,
  2006.

\bibitem{wang2014partition}
Lu~Wang, Yanghua Xiao, Bin Shao, and Haixun Wang.
\newblock How to partition a billion-node graph.
\newblock In {\em Data Engineering (ICDE), 2014 IEEE 30th International
  Conference on}, pages 568--579. IEEE, 2014.

\bibitem{yu2014useful}
Yi~Yu, Tengyao Wang, and Richard~J Samworth.
\newblock A useful variant of the davis--kahan theorem for statisticians.
\newblock {\em Biometrika}, 102(2):315--323, 2014.

\end{thebibliography}


\end{document}